\documentclass{article}
\usepackage[preprint]{colm2026_conference}

\usepackage{microtype}
\usepackage{hyperref}
\usepackage{url}
\usepackage{booktabs}
\usepackage{amsmath}
\usepackage{amsthm}
\usepackage{graphicx}
\usepackage{multirow}
\usepackage{wrapfig}
\usepackage{makecell}
\usepackage{xcolor}
\usepackage{enumitem}
\usepackage{mathtools}
\usepackage{subcaption}

\newcommand{\Sref}[1]{\S\ref{#1}}
\newcommand{\Tref}[1]{Table~\ref{#1}}
\newcommand{\Fref}[1]{Figure~\ref{#1}}
\newcommand{\Aref}[1]{Appendix~\ref{#1}}

\newcommand{\modeltag}[1]{
  \raisebox{0.15ex}{\textcolor{#1}{\rule{2.5ex}{1.0ex}}}
}

\newtheorem{theorem}{Theorem}

\newtheorem{innercustomthm}{Theorem}
\newenvironment{customthm}[1]
  {\renewcommand\theinnercustomthm{#1}\innercustomthm}
  {\endinnercustomthm}

\definecolor{ucColor}{RGB}{255,179,0}
\definecolor{ucimColor}{RGB}{250,120,40}
\definecolor{uclmColor}{RGB}{207,28,43}
\definecolor{uclmwfColor}{RGB}{235,75,155}
\definecolor{ucimlmColor}{RGB}{255,143,233}
\definecolor{ucimlmwfColor}{RGB}{152,18,176}
\definecolor{bcimlmColor}{RGB}{51,53,191}
\definecolor{bcimlmwfColor}{RGB}{70,168,204}

\usepackage{lineno}

\definecolor{darkblue}{rgb}{0, 0, 0.5}
\hypersetup{colorlinks=true, citecolor=darkblue, linkcolor=darkblue, urlcolor=darkblue}

\title{Differences in Text Generated by Diffusion and Autoregressive Language Models}

\author{Zeyang Zhang$^{1, 2}$\thanks{Equal contribution.} \quad
  Chengwei Liang$^{2}$\footnotemark[1] \quad
  Xingyan Chen$^{2}$\footnotemark[1] \quad
  Meiqi Gu$^{2}$\footnotemark[1] \\
  \textbf{Minrui Luo}$^{2}$ \quad
  \textbf{Jingzhao Zhang}$^{2, 1, 3}$\thanks{Corresponding authors.} \quad
  \textbf{Tianxing He}$^{2, 3, 1}$\footnotemark[2] \\
  $^1$Shanghai Qi Zhi Institute \\
  $^2$Institute for Interdisciplinary Information Sciences, Tsinghua University \\
  $^3$Xiongan AI Institute \\
  \texttt{\{zhangzey24,lcw24,chenxy24,gumq24,luomr22\}@mails.tsinghua.edu.cn} \\
  \texttt{\{jingzhaoz,hetianxing\}@mail.tsinghua.edu.cn}
}

\begin{document}

\ifcolmsubmission
\linenumbers
\fi

\maketitle

\begin{abstract}

Diffusion language models (DLMs) are promising alternatives to autoregressive language models (ARMs),
yet the intrinsic differences in their generated text remain underexplored.
We first find empirically that off-the-shelf DLMs exhibit lower $n$-gram entropy, higher semantic coherence, and higher semantic diversity.
To understand the cause, we conduct controlled experiments that decouple the effects of training objectives and decoding algorithms.
Results suggest that the DLM training objective contributes to the increases in semantic coherence and semantic diversity, but has a minor influence on entropy.
These differences are primarily driven by the bidirectional context; other components in the training objective, such as input masking, label masking, and the weighting function, have a much weaker influence.
Further, our experiments demonstrate that the reduction in entropy stems from DLMs' decoding algorithms, particularly confidence-based remasking strategies.
We provide a theoretical understanding for this entropy reduction phenomenon.
Together, our work uncovers key mechanisms underlying the differences between DLMs and ARMs in text generation, and informs future design of training objectives and decoding algorithms in DLMs.\footnote{Our code is available at \url{https://github.com/imzzy1201/DifferenceDiffusionAutoregressive}.}

\end{abstract}
\section{Introduction}
Diffusion language models (DLMs) have recently emerged as promising alternatives to autoregressive language models (ARMs) for natural language generation~\citep{li2025survey,tseng2025diffusion}. They offer distinct advantages, such as higher data efficiency and faster text generation~\citep{prabhudesai2025diffusion, khanna2025mercury,liu2025sdlm}.  With comparable model sizes and training budgets,
DLMs achieve performance competitive with ARMs across various tasks,
such as language understanding, mathematical reasoning, and code generation~\citep{song2025seed,khanna2025mercury,zhu2025llada}. 

However, beyond comparisons on efficiency and benchmark performance, the differences in the intrinsic properties of text generated by DLMs and ARMs remain underexplored.
In this work, we evaluate the generated texts using three metrics: entropy, semantic coherence, and semantic diversity (defined in \Sref{sec:metrics}), which capture text properties at the token, sentence, and document levels, respectively. As shown in \Fref{fig:off-the-shelf}, our evaluation of 20 off-the-shelf DLMs and ARMs reveals that DLMs tend to produce texts with lower $n$-gram entropy, higher semantic coherence, and higher semantic diversity ($n$-gram entropy serves as a proxy for full entropy; see \Sref{sec:motivation} for details). These findings provide insights into how the generation behaviors of DLMs and ARMs differ.

\begin{wrapfigure}[21]{r}{7cm}
\begin{center}
\vspace{-\baselineskip}
\includegraphics[width=0.9\linewidth]{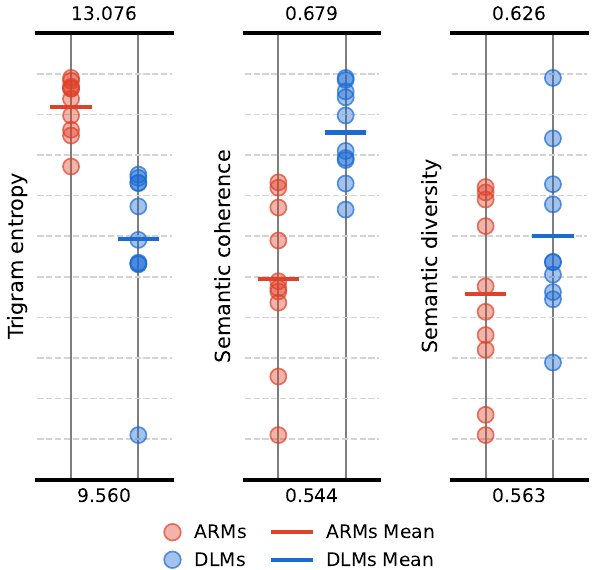}
\end{center}
\caption{Comparison of three metrics between 20 off-the-shelf DLMs and ARMs. DLMs tend to exhibit lower trigram entropy alongside higher semantic coherence and semantic diversity compared to ARMs.}
\label{fig:off-the-shelf}
\end{wrapfigure}

We conduct controlled experiments to isolate and analyze the mechanisms underlying these differences. Unlike ARMs, DLMs model the language distribution through mask prediction from noised inputs $p_\theta(x|x_{\text{noised}})$~\citep{austin2021structured,shi2024simplified,sahoo2024simple}. Moreover, DLMs employ a variety of flexible decoding strategies~\citep{nie2025large,cheng2025sdar,wu2025fast,bie2025llada2,ye2025dream}. Consequently, both the training objective and the decoding algorithm could contribute to the observed metric differences. 

To decouple their effects, we perform two sets of experiments: (1) on the training objectives, comparing DLMs and ARMs under a shared decoding setup; and (2) on the decoding algorithms, comparing various decoding strategies within DLMs. To summarize, our contributions are as follows:

\begin{itemize}[leftmargin=*]
\item We evaluate the generation behaviors of 20 off-the-shelf DLMs and ARMs on $n$-gram entropy, semantic coherence, and semantic diversity, providing insights into the differences between their generated corpora (\Sref{sec:motivation}).

\item Through controlled experiments on training objectives, we find that the DLM training objective contributes to the increase in semantic coherence and semantic diversity, but minimally affects entropy. By interpolating the objectives between DLMs and ARMs, we identify bidirectional context as the dominant factor responsible for these increments, whereas other components, such as input masking, label masking, or weighting functions, have a much smaller effect (\Sref{sec:train}).

\item Through detailed experiments and analyses on decoding algorithms, we demonstrate that the choice of remasking strategy is critical.
Specifically, DLMs' entropy drops under confidence-based remasking strategies, resulting in lower entropy than ARMs; under other strategies, DLMs' entropy remains approximately unchanged and stays at ARMs' level.
We also provide theoretical understanding for this effect (\Sref{sec:inference}).
\end{itemize}

\section{Related work}

\paragraph{Diffusion language models.}

Inspired by the success of diffusion models in image generation~\citep{rombach2022high,ho2020denoising}, early work establishes a viable path from continuous to discrete domains, introducing diffusion over discrete tokens for language modeling~\citep{austin2021structured,lou2023discrete}.
Building on subsequent theoretical developments~\citep{shi2024simplified,ou2024your,sahoo2024simple}, recent efforts scale these methods to large language models, such as LLaDA~\citep{nie2025large} and Dream~\citep{ye2025dream}, demonstrating strong scalability and promising reasoning capabilities. 

To further improve efficiency, various approaches are proposed, including autoregressive initialization and block-wise diffusion~\citep{wu2025fast1,ye2025dream,gong2024scaling,cheng2025sdar,wu2025fast}. Concurrently, the flexible decoding order of DLMs motivates specialized decoding strategies, 
including random selection, confidence-based or entropy-based remasking~\citep{chang2022maskgit,nie2025large,kim2025train,ye2025dream,ben2025accelerated}, dynamic scheduling~\citep{wu2025fast1,cheng2025sdar,wu2025fast,bie2025llada2}, and other related techniques. 
\paragraph{Comparison between DLMs and ARMs.}

Prior works comparing DLMs and ARMs can be organized into three aspects: training objective and dynamics, decoding mechanism, and downstream task performance.

In terms of training objective and dynamics, prior works examine how different training formulations shape model behavior, 
such as in representation structure, generalization, and robustness. Specifically, DLMs may exhibit early-layer redundancy and reduced recency bias~\citep{goel2026skip}, achieve higher data efficiency~\citep{prabhudesai2025diffusion,gao2025makes,ni2025diffusion}, and mitigate the reversal curse~\citep{kitouni2024factorization,shin2026understanding}. 

Regarding decoding mechanism, DLMs enable flexible, any-order generation and theoretically support more efficient parallel sampling~\citep{li2025breaking,jiang2025diffusion}. However, empirical results show that these advantages are not always realized in practice~\citep{yang2025powerful,ni2026flexibility}, 
as performance could be influenced by decoding strategies.
Other works discover that DLMs often determine their answers at early decoding stages~\citep{li2025diffusion}, and can be halted early~\citep{vaina2023diffusion}.  Diffusion-specific safety issues are also explored~\citep{rahimi2026step,zhang2025jailbreaking,wen2025devil}.

For downstream task performance, prior works generally find that DLMs are competitive with similarly sized ARMs~\citep{nie2024scaling,li2025survey}, and that DLMs show advantages in tasks requiring bidirectional reasoning or global context understanding~\citep{zhang2025diffusion,wang2026analyzing,xiong2025unveiling}. Further, some studies also demonstrate trade-offs between generation quality and generation speed~\citep{feng2025theoretical,dikov2025diffusion}.

Overall, existing studies primarily characterize differences in efficiency and downstream task performance, leaving the intrinsic properties of the generated text itself underexplored.
\section{Preliminaries and results for off-the-shelf models}
\label{sec:preliminaries}

\subsection{Diffusion language models}
\label{sec:diffusion}

Let $\mathcal{V}$ be the vocabulary. A text sequence of length $L$ is denoted by $x \in \mathcal{V}^L$, where $x^i$ is the $i$-th token. We denote by $\mathcal{D}_{\text{data}}$ the data distribution over $\mathcal{V}^L$.

DLMs learn to reconstruct clean text from noised inputs by predicting the ground-truth tokens for randomly masked positions~\citep{austin2021structured,nie2025large}:
\begin{equation}
\mathcal L(\theta) =-\mathbb E_{x_0,t,x_t} \left[ \omega_t \sum_{i=1}^{L} \mathbf 1[x_t^i = \texttt{MASK}] \log p_\theta(x_0^i \mid x_t) \right].
\end{equation}
Here, a clean sequence $x_0 \sim \mathcal{D}_{\text{data}}$ is noised into $x_t$ by converting each token to $\texttt{MASK}$ token with probability $t$, where $t \sim \text{Uniform}(0,1)$ represents the noise level. A weighting function $\omega_t$, commonly set to $1/t$, balances the loss across noise levels. The objective trains the model to recover the original token $x_0^i$ at each masked position in $x_t$.

Unlike the left-to-right sequential decoding of ARMs, DLMs start generation with a sequence of $\texttt{MASK}$ tokens, iteratively refining it into actual text. 
A common practice is block-wise decoding~\citep{arriola2025block}, where the sequence is divided into equal-length blocks that are refined consecutively. We focus on this setting.

Within each block, the model performs multiple denoising steps to gradually decode all masked positions before moving to the next block. At each step, the model predicts a token for every masked position in the block, then remasks a subset of them, i.e., reverting them to $\texttt{MASK}$, before proceeding to the next step. The choice of which positions to remask is determined by a \emph{remasking strategy}. Assume a block length of $B$ and a target number of $N$ denoising steps.
In this work, we investigate four commonly used remasking strategies:
\begin{itemize}[leftmargin=*]
    \item \textbf{Low-confidence remasking}~\citep{chang2022maskgit,nie2025large}: This strategy remasks tokens with the lowest \emph{confidence} values, i.e., their predicted probabilities, leaving exactly $B/N$ most confident tokens to be decoded at each step.
    \item \textbf{Dynamic low-confidence remasking}~\citep{cheng2025sdar,wu2025fast,bie2025llada2}: 
    At each step, this strategy first attempts to remask tokens with confidence lower than $\tau$. If the number of decoded tokens does not reach the target $B/N$, it falls back to low-confidence remasking. Consequently, the total number of decoding steps for a block may be fewer than $N$, since more than $B/N$ tokens may be decoded in a single step.
    \item \textbf{High-entropy remasking}~\citep{ye2025dream}: This strategy evaluates the entropy of the predicted token distribution at each position, and remasks positions with the highest entropy values so that exactly $B/N$ tokens are decoded at each step.
    \item \textbf{Random remasking}~\citep{nie2025large}: This strategy uniformly randomly selects tokens to remask so that exactly $B/N$ tokens are decoded at each step.
\end{itemize}
 
We refer to the first two approaches collectively as \emph{confidence-based remasking strategies}, as they select remasking positions based on the confidence values of the predicted tokens.

\subsection{Metrics for comparing generated text}
\label{sec:metrics}

We aim to understand the different generation behaviors of DLMs and ARMs across multiple granularities, progressing from the smallest token level, to the intermediate sentence level, and finally to the overall document level. Accordingly, we adopt three metrics, entropy, semantic coherence, and semantic diversity. The underlying intuition motivating this choice is that each metric captures how the generated corpus ``disperses'' at its corresponding granularity. Specifically, higher entropy, lower semantic coherence, and higher semantic diversity reflect more dispersed generation behaviors at the token, sentence, and document levels, respectively.

Given $p_{\mathcal G}$ as the generation distribution over the tuple $(c, x)$, where $c$ is the prompt and $x$ is the response, the three metrics are defined as follows:
\begin{itemize}[leftmargin=*]
\item \textbf{Entropy} of the generation distribution represents the average uncertainty per token:
\begin{equation}
\label{eq:entropy}
\mathcal H(p_{\mathcal G}) \coloneqq -\mathbb E_{(c, x) \sim p_{\mathcal G}} \left[\frac{1}{L}\log p_{\mathcal G}(x \mid c)\right].
\end{equation}

When $p_{\mathcal G}(x \mid c)$ is intractable to evaluate, following \citet{shannon1951prediction}, we use \emph{$n$-gram entropy}, i.e., the entropy of the empirical $n$-gram distribution of the corpus, as a proxy for the full entropy.

\item \textbf{Semantic coherence} measures the fluency of semantic flow. Following \citet{parola2023speech} and \citet{bedi2015automated}, we evaluate this metric by calculating the average cosine similarity between embeddings of adjacent sentences. Given a normalized text embedding model $\text{emb}(\cdot)$\footnote{All sentence embeddings satisfy $\|\mathrm{emb}(\cdot)\|_2 = 1$.}, and denoting each sampled response as $K$ consecutive sentences $x=s^1s^2\cdots s^K$, we define semantic coherence as:
\begin{equation}
\text{Coh}(p_{\mathcal G}) \coloneqq \mathbb E_{(c, x) \sim p_{\mathcal G}} \left[ \frac{1}{K-1} \sum_{i=1}^{K-1} \langle \text{emb}(s^i), \text{emb}(s^{i+1}) \rangle  \;\middle|\; K \geq 2 \right]. ],\quad K \geq 2
\end{equation}
\item \textbf{Semantic diversity} reflects the overall creativity of the generated corpus.
In line with \citet{kirk2023understanding}, we quantify it as one minus the average pairwise cosine similarity of the generated sequences in the semantic embedding space.
Using the same normalized embedding model, it is computed as:
\begin{equation}
\text{Div}(p_{\mathcal G}) \coloneqq 1 - \mathbb E_{(c, x), (c', x') \sim p_{\mathcal G}} [\langle \text{emb}(x), \text{emb}(x') \rangle].
\end{equation}
This is equivalent to the trace of the covariance matrix of the embeddings, i.e., $\text{tr}\left(\text{Cov}_{(c, x) \sim p_{\mathcal G}}(\text{emb}(x))\right)$.\footnote{Let $z=\text{emb}(x)$ with $\|z\|_2=1$. Then $\mathbb E[\langle z,z'\rangle]=\|\mathbb E[z]\|_2^2$ for independent $z,z'$, and $\mathrm{tr}(\mathrm{Cov}(z))=\mathbb E[\|z\|_2^2]-\|\mathbb E[z]\|_2^2=1-\|\mathbb E[z]\|_2^2$, yielding the equivalence.}
\end{itemize}

\subsection{Evaluation on off-the-shelf models}
\label{sec:motivation}

To establish a practical foundation for our work and motivate our subsequent analysis, we begin by evaluating the generation behaviors of off-the-shelf DLMs and ARMs across the three metrics in \Sref{sec:metrics}. We compare 10 DLMs and 10 ARMs from diverse model families, with comparable model sizes ranging from 3B to 8B (listed in \Aref{app:offtheshelf}). To estimate these metrics, we sample texts from these models on 1000 different prompts from the Fineweb dataset~\citep{fineweb}, using their default sampling configurations. 

Regarding entropy, the non-sequential generation nature of DLMs prevents computing the log probability of a sequence;
thus, we adopt $n$-gram entropy as a proxy \citep{shannon1951prediction}. To evaluate semantic coherence and semantic diversity, we extract embeddings using \texttt{bge-m3}~\citep{multi2024m3}. As shown in \Fref{fig:off-the-shelf}, DLMs on average exhibit lower trigram entropy, higher semantic coherence, and higher semantic diversity than ARMs. Unigram and bigram entropy exhibit the same trends as trigram entropy. See \Aref{app:offtheshelf} for more details.

While the results on off-the-shelf models reveal clear trends, there is no strict control over their training setups, which may compromise the robustness of the observed phenomena and hinder rigorous analysis.
Therefore, in \Sref{sec:train} and \Sref{sec:inference}, we train models from scratch and conduct controlled experiments to isolate and analyze the mechanisms underlying these differences. Since DLMs and ARMs differ in both their training objectives and decoding algorithms, we decouple these effects through two sets of experiments: in \Sref{sec:train}, we compare the training objectives under a shared decoding scheme; in \Sref{sec:inference}, we evaluate various decoding strategies on the same model.
\section{Controlled experiments on training objectives}
\label{sec:train}

As initially observed in off-the-shelf models (\Fref{fig:off-the-shelf}), DLMs and ARMs exhibit distinct generation behaviors. We first investigate whether these differences are intrinsically driven by their training objectives.

\subsection{Decomposition of objective components}
\label{sec:decomposition}

To analyze the intrinsic effects of training objectives, we begin by decomposing each objective into its constituent components.
Recall from \Sref{sec:diffusion} that the objective of DLMs is 
\begin{equation}
\mathcal L_{\text{DLM}}(\theta) \coloneqq-\mathbb E_{x_0,t,x_t} \left[ \omega_t \sum_{i=1}^{L} \mathbf 1[x_t^i = \texttt{MASK}] \log p_\theta(x_0^i \mid x_t) \right],\quad \text{where } \omega_t=1/t.
\end{equation}
We can express ARM's training objective in a similar format:
\begin{equation}
\mathcal L_{\text{ARM}}(\theta) \coloneqq-\mathbb E_{x_0} \left[ \sum_{i=1}^{L} \log p_\theta(x_0^i \mid x_0^{<i}) \right] . 
\end{equation}
The differences between these two objectives can be decomposed into four components:
\begin{itemize}[leftmargin=*]
    \item \textbf{Input masking}. DLMs take the noised sequence $x_{\text{input}}=x_t$ as input, while ARMs take the raw sequence $x_{\text{input}}=x_0$. 
    \item \textbf{Label masking}. DLMs accumulate the loss only over masked tokens ($\mathbf{1}[x_t^i = \texttt{MASK}]$), while ARMs accumulate it over all tokens. 
    \item \textbf{Weighting function}. The diffusion process of DLMs yields a time-dependent weight ($\omega_t=1/t$), whereas ARMs use a constant weight. 
    \item \textbf{Context scope}. DLMs condition on bidirectional context $p_\theta(x_0^i \mid x_{\text{input}})$, whereas ARMs condition on unidirectional context $p_\theta(x_0^i \mid x_{\text{input}}^{<i})$.
\end{itemize}

To investigate how each individual component affects the resulting model, we design six intermediate training objectives, as detailed in \Fref{fig:trainingresults}. See \Aref{app:trainobj} for more details on the choice of training objectives. These objectives serve as an interpolation between the ARM objective and the DLM objective. Starting from the ARM objective, we gradually change each component until we obtain the DLM objective. Consequently, by observing at which step the differences in metrics emerge, we can isolate the influence of specific components.

For instance, the \texttt{uc+im} objective in \Fref{fig:trainingresults} takes a randomly masked prefix $x_t^{<i}$ as input and is trained to predict the ground truth $x_0^i$ at all positions, using the loss $-\mathbb E_{x_0,t,x_t} [\sum_{i=1}^{L} \log p_\theta(x_0^i \mid x_t^{<i})]$. Building on this, the \texttt{uc+im+lm} objective trains on $x_0^i$ only for $x_t^i = \texttt{MASK}$, utilizing the loss $-\mathbb E_{x_0,t,x_t} [\sum_{i=1}^{L} \mathbf 1[x_t^i = \texttt{MASK}] \log p_\theta(x_0^i \mid x_t^{<i})]$.

\begin{figure}[t]
    \begin{center}
    \includegraphics[width=\linewidth]{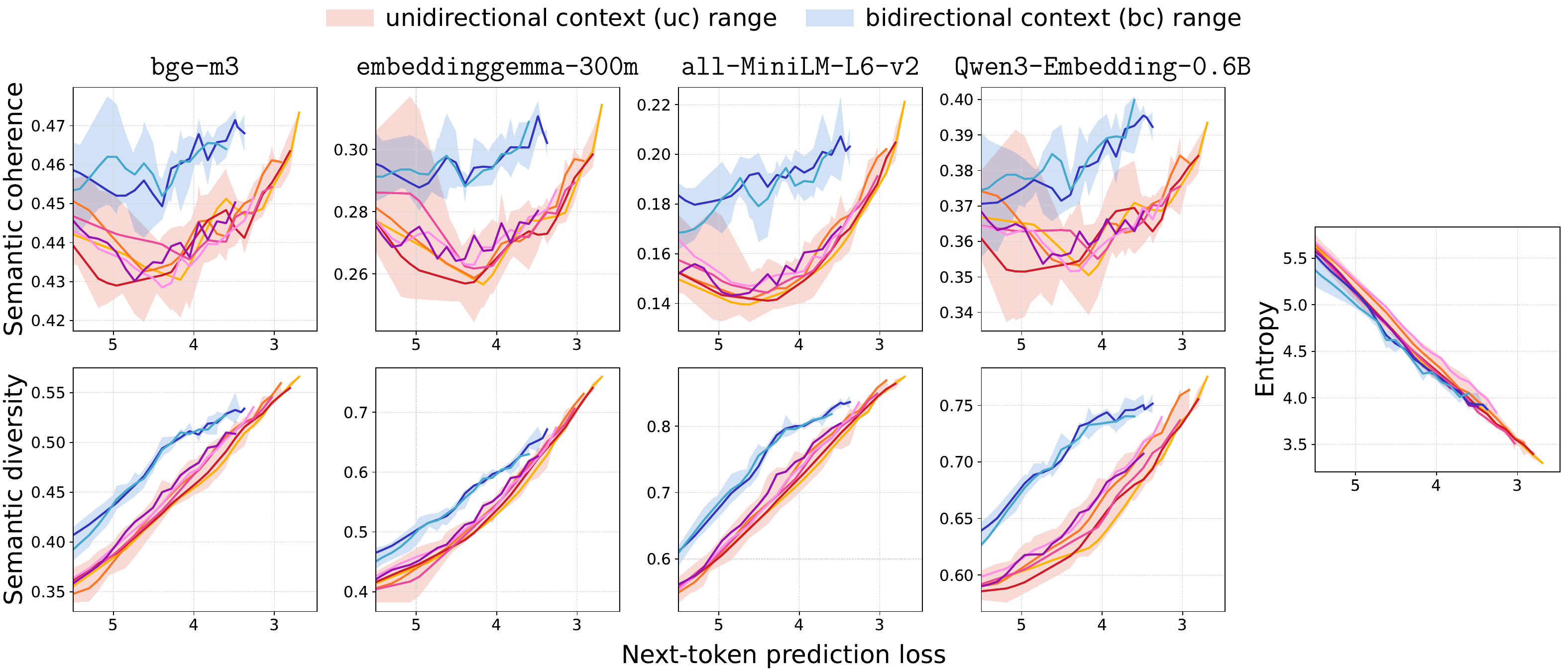}
    \end{center}
    \begin{center}
    \renewcommand{\arraystretch}{1.34}
    \begin{tabular}{cll}
    \toprule
    \multicolumn{1}{c}{\bf Legend color}  &\multicolumn{1}{c}{\bf Components}  &\multicolumn{1}{c}{\bf Training objective} \\
    \midrule
    \modeltag{ucColor}{} & \texttt{uc} & $-\mathbb E_{x_0\phantom{,t,x_t}} [\ \phantom{\frac{1}{t}} \sum_{i=1}^{L} \phantom{1[x_t^i = \texttt{MASK}]} \log p_\theta(x_0^i \mid x_0^{<i}) \ ] = \mathcal{L}_{\text{ARM}}$ \\ 
    \modeltag{ucimColor}{} & \texttt{uc+im} &  $-\mathbb E_{x_0,t,x_t} [\ \phantom{\frac{1}{t}} \sum_{i=1}^{L} \phantom{1[x_t^i = \texttt{MASK}]} \log p_\theta(x_0^i \mid x_t^{<i}) \ ]$ \\ 
    \modeltag{uclmColor}{} & \texttt{uc+lm} & $-\mathbb E_{x_0,t,x_t} [\ \phantom{\frac{1}{t}} \sum_{i=1}^{L} \mathbf 1[x_t^i = \texttt{MASK}] \log p_\theta(x_0^i \mid x_0^{<i}) \ ]$ \\ 
    \modeltag{uclmwfColor}{} & \texttt{uc+lm+wf} & $-\mathbb E_{x_0,t,x_t} [\ \frac{1}{t} \sum_{i=1}^{L} \mathbf 1[x_t^i = \texttt{MASK}] \log p_\theta(x_0^i \mid x_0^{<i}) \ ]$ \\ 
    \modeltag{ucimlmColor}{} & \texttt{uc+im+lm} & $-\mathbb E_{x_0,t,x_t} [\ \phantom{\frac{1}{t}} \sum_{i=1}^{L} \mathbf 1[x_t^i = \texttt{MASK}] \log p_\theta(x_0^i \mid x_t^{<i}) \ ]$ \\ 
    \modeltag{ucimlmwfColor}{} & \texttt{uc+im+lm+wf} & $-\mathbb E_{x_0,t,x_t} [\ \frac{1}{t} \sum_{i=1}^{L} \mathbf 1[x_t^i = \texttt{MASK}] \log p_\theta(x_0^i \mid x_t^{<i}) \ ]$ \\ 
    \modeltag{bcimlmColor}{} & \texttt{bc+im+lm} & $-\mathbb E_{x_0,t,x_t} [\ \phantom{\frac{1}{t}} \sum_{i=1}^{L} \mathbf 1[x_t^i = \texttt{MASK}] \log p_\theta(x_0^i \mid x_t) \;\ \,\, ]$ \\ 
    \modeltag{bcimlmwfColor}{} & \texttt{bc+im+lm+wf} & $-\mathbb E_{x_0,t,x_t} [\ \frac{1}{t} \sum_{i=1}^{L} \mathbf 1[x_t^i = \texttt{MASK}] \log p_\theta(x_0^i \mid x_t) \;\ \,\, ] = \mathcal{L}_{\text{DLM}}$ \\ 
    \bottomrule
    \end{tabular}
    \end{center}
    \caption{Definition and evaluation of eight interpolated training objectives. Abbreviations denote the incorporated components: unidirectional context (\texttt{uc}), bidirectional context (\texttt{bc}), input masking (\texttt{im}), label masking (\texttt{lm}), and weighting function (\texttt{wf}). Curves are averaged over random seeds. A clear clustering emerges, demonstrating that the bidirectional context in the DLM training objective markedly increases semantic coherence and semantic diversity, while entropy stays approximately the same across all eight models.}
    \label{fig:trainingresults}
\end{figure}

\subsection{Experimental setup}
\label{sec:train_experimental_setup}

We train eight models from scratch, each optimized with one of the training objectives defined in \Fref{fig:trainingresults}. All models use an identical 120M-parameter LLaMA architecture~\citep{touvron2023llama}. Each model is trained on the Fineweb~\citep{fineweb} dataset for 2B tokens using a context length of 512. Each experiment is repeated with three different random seeds.

For evaluation, to isolate the effect of decoding algorithms, we adopt a unified decoding procedure across all models. Specifically, we evaluate every model using standard sequential autoregressive decoding. For DLMs, this setup is equivalent to setting the decoding block length to one. At each checkpoint, we evaluate the models by generating 512 sequences, each 512 tokens long, conditioned on distinct 32-token prompts sampled from the evaluation split of the dataset. 

Due to the sequential decoding, we can directly estimate entropy by autoregressive factorization:
\begin{equation}
\label{eq:arfactorize}
\mathcal{H}(p_\mathcal{G})=\mathbb{E}_{(c,x)\sim p_\mathcal{G}}\left[\frac{1}{L}\sum_{i=1}^L \log p_\mathcal{G}(x^i|x^{<i},c)\right].
\end{equation}
To ensure robustness, we evaluate both semantic coherence and semantic diversity using four different embedding models: \texttt{bge-m3}~\citep{multi2024m3}, \texttt{embeddinggemma-300m}~\citep{vera2025embeddinggemma}, \texttt{all-MiniLM-L6-v2}~\citep{reimers2021train}, and \texttt{Qwen3-Embedding-0.6B}~\citep{qwen3embedding}. See \Aref{app:trainsetup} for more details on experimental setups.

\subsection{Results}

\label{sec:ablation_results}

To ensure a fair comparison, we evaluate the models at checkpoints that achieve the same \emph{next-token prediction loss} on the validation set (see \Aref{app:trainsetup} for details). Since all models are restricted to sequential generation, this loss is exactly the cross-entropy between each model's generation distribution and the data distribution, and thus reflects generation quality. Consequently, this alignment ensures that models trained with different objectives have equivalent capabilities.

Under this controlled setup, we observe a clear clustering effect across the metrics (\Fref{fig:trainingresults}). The eight models are divided into two distinct groups. The cluster containing the DLM objective consistently achieves \textbf{higher semantic coherence and higher semantic diversity} than the cluster containing the ARM objective. \textbf{This matches the trends in \Fref{fig:off-the-shelf}.} Meanwhile, we observe \textbf{no significant difference in entropy} across the evaluated models.

We attribute this two-cluster separation to the difference in \emph{context scope}: those utilizing bidirectional context form the cluster with higher semantic coherence and higher semantic diversity, while those using unidirectional context form the other. This separation suggests that bidirectional context is the dominant factor driving these differences, whereas input masking, label masking, and the weighting function exert a much weaker influence.

Specifically, we emphasize two pairs of experiments in \Fref{fig:trainingresults}, where each pair differs only in context scope: (\texttt{uc+im+lm}, \texttt{bc+im+lm}) and (\texttt{uc+im+lm+wf}, \texttt{bc+im+lm+wf}). For each pair, the two experiments show a large gap in semantic coherence and semantic diversity, indicating the dominant effect of bidirectional context.

Last but not least, we verify that all phenomena in this section are robust across different datasets (Fineweb and TinyStories) and architectures (LLaMA and Qwen2). See \Aref{app:trainrobust} for details.

\section{Analyses on decoding algorithms}
\label{sec:inference}

Apart from left-to-right decoding, DLMs can generate tokens in the order guided by an arbitrary remasking strategy, as mentioned in \Sref{sec:diffusion}.
At the same time, the experiments in \Sref{sec:train} show that the training objective alone does not lead to an entropy difference between DLMs and ARMs, while \Fref{fig:off-the-shelf} points to potentially lower entropy in DLMs.
We therefore investigate the influence of remasking strategies on the generation behaviors of DLMs.

\subsection{Experimental setup}

We use sequential decoding as a baseline to evaluate the influence of these remasking strategies. Our evaluation builds on the configuration from \Sref{sec:train_experimental_setup}, using the final checkpoints of models trained in \Sref{sec:train} with the $\mathcal{L}_{\text{DLM}}$ objective, modifying only the decoding algorithm. For all strategies, we set the target number of denoising steps per block $N$ equal to the block length $B$, which targets decoding one token per step.
For the dynamic low-confidence remasking strategy, following \citet{cheng2025sdar}, we set the confidence threshold to $\tau=0.9$.
We evaluate three block lengths: 2, 8, and 32. Following \Sref{sec:train}, we conduct ablations across seeds, datasets and architectures. They are detailed in \Aref{app:inference_robust}.

\begin{figure}[t]
    \begin{center}
    \includegraphics[width=\linewidth]{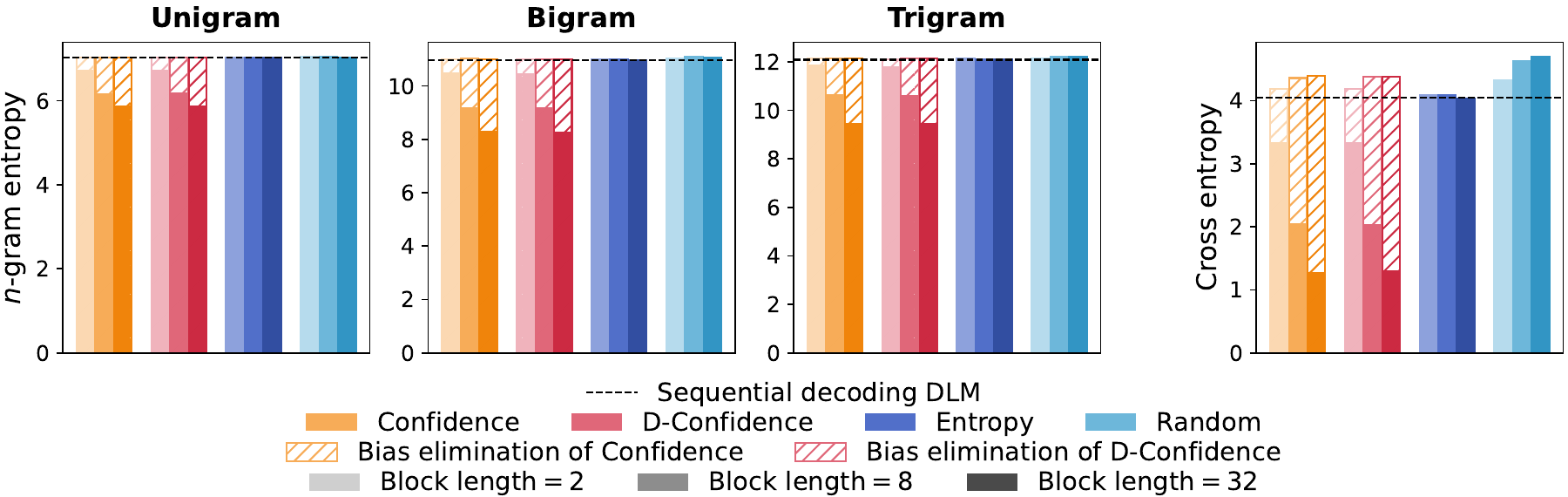}
    \end{center}
    \caption{$n$-gram entropy and cross-entropy (defined in \mbox{Eq.\eqref{eq:ce}}) across different DLM remasking strategies and block lengths. The labels denote low-confidence remasking (Confidence), dynamic low-confidence remasking (D-Confidence), high-entropy remasking (Entropy), random remasking (Random).
    The horizontal dashed line in the last subplot indicates $\mathcal{H}(p_{\text{seq}})$.
    The combined height of the bars for the confidence‑based strategies, including the hatched areas stacked on top, represents the results of the bias elimination experiments (refer to \Sref{point:bias_elimination}). This figure shows results for a single seed; see \Aref{app:inference_robust} for other seeds.}
    \label{fig:inference_entropy}
\end{figure}

\subsection{Results}
\label{sec:inference_results}

\paragraph{Entropy reduction phenomenon.} Following \Sref{sec:motivation}, we use $n$-gram entropy to approximate full entropy. As shown in \Fref{fig:inference_entropy}, confidence-based strategies consistently lower $n$-gram entropy compared to sequential decoding, whereas the other strategies have a comparatively minor influence. This suggests that \textbf{confidence-based decoding reduces entropy}.

To verify this, we establish the following inequality chain, comparing the full entropy under a confidence-based decoding strategy $\mathcal{H}(p_{\text{con}})$  with that under sequential decoding $\mathcal{H}(p_{\text{seq}})$:
\begin{equation}
\label{eq:inference_inq}
\mathcal H(p_{\text{con}}) \overset{\text{(i)}}{\leq} \mathcal{H} (p_{\text{con}}, p_{\text{seq}})\overset{\text{(ii)}}{\leq}\mathcal H(p_{\text{seq}}) -\delta,
\end{equation}

for some non-negligible $\delta>0$, where $\mathcal H(p_{\text{con}})$ and $\mathcal H(p_{\text{seq}})$ are defined in \mbox{Eq.\eqref{eq:entropy}}, and the intermediate term denotes the cross-entropy:
\begin{equation}
\label{eq:ce}
\mathcal H(p_{\text{con}},p_{\text{seq}}) \coloneqq -\mathbb E_{(c, x)\sim p_{\text{con}}} \left[\frac{1}{L}\log p_{\text{seq}}(x \mid c)\right].
\end{equation}
Inequality~(\hyperref[eq:inference_inq]{i}) holds by Gibbs' inequality~\citep{bremaud2012introduction}.
We empirically evaluate inequality~(\hyperref[eq:inference_inq]{ii}) by sampling $x$ from $p_{\text{con}}$ (or $p_{\text{seq}}$) for each prompt $c$, and computing $\log p_{\text{seq}}(x)$ using autoregressive factorization similar to \mbox{Eq.\eqref{eq:arfactorize}}. 

As shown in \Fref{fig:inference_entropy}, for confidence-based strategies, the cross-entropy $\mathcal{H}(p_{\text{con}},p_{\text{seq}})$ consistently falls far below the sequential decoding entropy $\mathcal H(p_{\text{seq}})$. This validates inequality~(\hyperref[eq:inference_inq]{ii}), completing the inequality chain. Conversely, inequality~(\hyperref[eq:inference_inq]{ii}) fails when replacing $p_{\text{con}}$ by other non-confidence-based strategies, thereby breaking the chain.\footnote{Note that the failure of inequality~(\hyperref[eq:inference_inq]{ii}) does not imply an increase in entropy, since \mbox{Eq.\eqref{eq:inference_inq}} does not hold in the reverse direction.}

\paragraph{Understanding of entropy reduction.}
One possible explanation for this entropy reduction is a \emph{distributional bias}. In theory, even for a perfectly trained DLM, confidence-based decoding yields a biased generation distribution.

As mentioned in \Sref{sec:diffusion}, confidence-based strategies first sample from the predicted distribution, then select remasking positions based on the confidence values of these sampling outcomes. By contrast, other strategies select positions based solely on the predicted distribution (i.e., by the distribution entropy or at random) before actual sampling occurs. Consequently, the outcome-dependence in confidence-based strategies introduces bias, whereas other strategies remain unbiased.

To further explore this explanation, we provide two supporting observations:

\begin{itemize}[leftmargin=*]
\item \textbf{Theoretical illustration under factorization assumption.}
To understand how this bias connects to entropy reduction, we first examine a single denoising step. Confidence-based strategies prefer positions with higher confidence values, i.e., predicted probabilities. As a result, they sharpen the predicted distribution by amplifying high probabilities and suppressing low ones, thereby reducing per-step entropy.

Building upon this, we then formally prove that entropy reduction holds for the final generation under a simplified factorization assumption (Theorem~\ref{theorem:maintext}).

\begin{customthm}{1}[Informal]
\label{theorem:maintext}
Assume that the data distribution at each position is independent of the others, so that the full distribution factorizes as $p_{\mathrm{data}}(x)=\prod_{i=1}^{L} p_{\mathrm{data}}^i(x^i)$, and that the DLM $p_\theta$ is optimally trained, i.e., for any masked position $i$ of a noised input $x_t$, $p_\theta(x_0^i \mid x_t)=p_{\mathrm{data}}^i(x_0^i)$. Let $p_{\mathrm{lcr}}$ and $p_{\mathrm{dlcr}}$ denote the generation distributions produced by applying low-confidence remasking and dynamic low-confidence remasking to the model, respectively. Then,
\begin{equation}
\mathcal H(p_{\mathrm{lcr}})\le \mathcal H(p_{\mathrm{data}}), \qquad
\mathcal H(p_{\mathrm{dlcr}})\le \mathcal H(p_{\mathrm{data}}).
\end{equation}
\end{customthm}
See Appendix~\ref{app:inference_theory} for the full statement and the proof.

\item \textbf{Bias elimination experiments.}\label{point:bias_elimination}
As an empirical validation, we apply a simple fix to the confidence-based strategy to eliminate this bias. After the strategy determines the positions to decode, we \emph{resample} the tokens at these positions. This modification breaks the correlation between the remasking position selection and the sampled token outcomes, thereby eliminating the bias. As shown in \Fref{fig:inference_entropy}, this fix effectively restores the $n$-gram entropy to a level similar to that of sequential decoding and invalidates inequality~(\hyperref[eq:inference_inq]{ii}). Details can be found in Appendix~\ref{app:inference_bias}.

\end{itemize}

\paragraph{Other influences.}

We further evaluate how remasking strategies affect semantic coherence and semantic diversity. Our results show that confidence-based strategies are relatively unstable: while they consistently increase semantic coherence across our multiple experiments, their influence on semantic diversity is highly dataset-dependent. Specifically, in our main experiments on the Fineweb dataset, confidence-based strategies increase semantic diversity, but ablation experiments on the TinyStories dataset show decreases in semantic diversity, suggesting a relatively unpredictable behavior. 

In contrast, other strategies have comparatively negligible effects on both metrics. For confidence-based strategies, applying bias elimination also removes their effects. These results suggest that the bias-based distinction may also determine whether a strategy affects semantic coherence and semantic diversity. See Appendix~\ref{app:inference_other} for more details.

\section{Conclusion and discussion}

In this work, we investigate the different generation behaviors of DLMs and ARMs by decoupling the effects of their training objectives and decoding procedures. Our control experiments identify two key mechanisms: First, the DLM training objective, particularly through the bidirectional context, drives the model toward higher semantic coherence and higher semantic diversity, while maintaining an entropy similar to that of ARMs. Second, the DLM decoding algorithms, specifically confidence-based remasking strategies, reduce entropy relative to autoregressive decoding. 

It is also worth noting that our observations imply that DLMs can achieve higher semantic diversity at the same or lower entropy levels than ARMs. This suggests that DLMs may relax the conventional text generation trade-off between diversity and certainty, pointing to their potential for diverse yet controllable text generation.
Together, our findings offer insights into the selection between DLMs and ARMs for different scenarios and may inform the design of training objectives and decoding algorithms in DLMs.

\section{Limitation and future work}

We conclude by outlining several limitations and directions for future work. Our experiments are limited to relatively small models; it remains to be verified that our findings generalize to larger models and datasets. Our theoretical analysis of the decoding procedure relies on several simplifying assumptions. Future work may relax these assumptions and extend the study to broader DLM architectures and decoding algorithms.

\section*{Acknowledgements}
We thank Haoming Wen from Tsinghua University for his valuable feedback and for his constructive suggestions on writing.

\bibliography{references}

@article{shannon1951prediction,
  title={Prediction and entropy of printed English},
  author={Shannon, Claude E},
  journal={Bell system technical journal},
  volume={30},
  number={1},
  pages={50--64},
  year={1951},
  publisher={Wiley Online Library}
}

@article{multi2024m3,
  title={Bge m3-embedding: Multi-lingual, multi-functionality, multi-granularity text embeddings through self-knowledge distillation},
  author={Chen, Jianlv and Xiao, Shitao and Zhang, Peitian and Luo, Kun and Lian, Defu and Liu, Zheng},
  journal={arXiv preprint arXiv:2402.03216},
  volume={4},
  number={5},
  year={2024}
}

@article{nie2025large,
  title={Large Language Diffusion Models},
  author={Nie, Shen and Zhu, Fengqi and You, Zebin and Zhang, Xiaolu and Ou, Jingyang and Hu, Jun and Zhou, Jun and Lin, Yankai and Wen, Ji-Rong and Li, Chongxuan},
  journal={arXiv preprint arXiv:2502.09992},
  year={2025}
}

@article{liu2025sdlm,
  title={Sequential Diffusion Language Models},
  author={Liu, Yangzhou and Cao, Yue and Li, Hao and Luo, Gen and Chen, Zhe and Wang, Weiyun and Liang, Xiaobo and Qi, Biqing and Wu, Lijun and Tian, Changyao and Zhang, Yanting and Li, Yuqiang and Lu, Tong and Qiao, Yu and Dai, Jifeng and Wang, Wenhai},
  journal={arXiv preprint arXiv:2509.24007},
  year={2025}
}

@article{ye2025dream,
  title={Dream 7b: Diffusion large language models},
  author={Ye, Jiacheng and Xie, Zhihui and Zheng, Lin and Gao, Jiahui and Wu, Zirui and Jiang, Xin and Li, Zhenguo and Kong, Lingpeng},
  journal={arXiv preprint arXiv:2508.15487},
  year={2025}
}

@article{gong2024scaling,
  title={Scaling diffusion language models via adaptation from autoregressive models},
  author={Gong, Shansan and Agarwal, Shivam and Zhang, Yizhe and Ye, Jiacheng and Zheng, Lin and Li, Mukai and An, Chenxin and Zhao, Peilin and Bi, Wei and Han, Jiawei and others},
  journal={arXiv preprint arXiv:2410.17891},
  year={2024}
}

@article{wu2025fast,
  title={Fast-dllm v2: Efficient block-diffusion llm},
  author={Wu, Chengyue and Zhang, Hao and Xue, Shuchen and Diao, Shizhe and Fu, Yonggan and Liu, Zhijian and Molchanov, Pavlo and Luo, Ping and Han, Song and Xie, Enze},
  journal={arXiv preprint arXiv:2509.26328},
  year={2025}
}

@article{cheng2025sdar,
  title={Sdar: A synergistic diffusion-autoregression paradigm for scalable sequence generation},
  author={Cheng, Shuang and Bian, Yihan and Liu, Dawei and Zhang, Linfeng and Yao, Qian and Tian, Zhongbo and Wang, Wenhai and Guo, Qipeng and Chen, Kai and Qi, Biqing and others},
  journal={arXiv preprint arXiv:2510.06303},
  year={2025}
}

@inproceedings{chang2022maskgit,
  title={Maskgit: Masked generative image transformer},
  author={Chang, Huiwen and Zhang, Han and Jiang, Lu and Liu, Ce and Freeman, William T},
  booktitle={Proceedings of the IEEE/CVF conference on computer vision and pattern recognition},
  pages={11315--11325},
  year={2022}
}

@article{bie2025llada2,
  title={Llada2. 0: Scaling up diffusion language models to 100b},
  author={Bie, Tiwei and Cao, Maosong and Chen, Kun and Du, Lun and Gong, Mingliang and Gong, Zhuochen and Gu, Yanmei and Hu, Jiaqi and Huang, Zenan and Lan, Zhenzhong and others},
  journal={arXiv preprint arXiv:2512.15745},
  year={2025}
}

@article{touvron2023llama,
  title={Llama: Open and efficient foundation language models},
  author={Touvron, Hugo and Lavril, Thibaut and Izacard, Gautier and Martinet, Xavier and Lachaux, Marie-Anne and Lacroix, Timoth{\'e}e and Rozi{\`e}re, Baptiste and Goyal, Naman and Hambro, Eric and Azhar, Faisal and others},
  journal={arXiv preprint arXiv:2302.13971},
  year={2023}
}

@article{loshchilov2017decoupled,
  title={Decoupled weight decay regularization},
  author={Loshchilov, Ilya and Hutter, Frank},
  journal={arXiv preprint arXiv:1711.05101},
  year={2017}
}

@article{vera2025embeddinggemma,
  title={Embeddinggemma: Powerful and lightweight text representations},
  author={Vera, Henrique Schechter and Dua, Sahil and Zhang, Biao and Salz, Daniel and Mullins, Ryan and Panyam, Sindhu Raghuram and Smoot, Sara and Naim, Iftekhar and Zou, Joe and Chen, Feiyang and others},
  journal={arXiv preprint arXiv:2509.20354},
  year={2025}
}

@article{reimers2021train,
  title={Train the best sentence embedding model ever with 1b training pairs},
  author={Reimers, Niels and Omar, E and Joao, G and Tom, A},
  journal={Community week using JAX/Flax for NLP CV. HugginFace},
  volume={6},
  year={2021}
}

@article{qwen3embedding,
  title={Qwen3 embedding: Advancing text embedding and reranking through foundation models},
  author={Zhang, Yanzhao and Li, Mingxin and Long, Dingkun and Zhang, Xin and Lin, Huan and Yang, Baosong and Xie, Pengjun and Yang, An and Liu, Dayiheng and Lin, Junyang and others},
  journal={arXiv preprint arXiv:2506.05176},
  year={2025}
}

@article{fineweb,
  title={The fineweb datasets: Decanting the web for the finest text data at scale},
  author={Penedo, Guilherme and Kydl{\'\i}{\v{c}}ek, Hynek and Lozhkov, Anton and Mitchell, Margaret and Raffel, Colin and Von Werra, Leandro and Wolf, Thomas and others},
  journal={Advances in Neural Information Processing Systems},
  volume={37},
  pages={30811--30849},
  year={2024}
}

@article{li2025survey,
  title={A survey on diffusion language models},
  author={Li, Tianyi and Chen, Mingda and Guo, Bowei and Shen, Zhiqiang},
  journal={arXiv preprint arXiv:2508.10875},
  year={2025}
}

@article{tseng2025diffusion,
  title={Diffusion-based Large Language Models Survey},
  author={Tseng, Chiung-Yi and Zhang, Danyang and Song, Junhao and Bi, Ziqian},
  journal={Authorea Preprints},
  year={2025},
  publisher={Authorea}
}

@article{qwen3,
    title={Qwen3 Technical Report}, 
    author={An Yang and Anfeng Li and Baosong Yang and Beichen Zhang and Binyuan Hui and Bo Zheng and Bowen Yu and Chang Gao and Chengen Huang and Chenxu Lv and Chujie Zheng and Dayiheng Liu and Fan Zhou and Fei Huang and Feng Hu and Hao Ge and Haoran Wei and Huan Lin and Jialong Tang and Jian Yang and Jianhong Tu and Jianwei Zhang and Jianxin Yang and Jiaxi Yang and Jing Zhou and Jingren Zhou and Junyang Lin and Kai Dang and Keqin Bao and Kexin Yang and Le Yu and Lianghao Deng and Mei Li and Mingfeng Xue and Mingze Li and Pei Zhang and Peng Wang and Qin Zhu and Rui Men and Ruize Gao and Shixuan Liu and Shuang Luo and Tianhao Li and Tianyi Tang and Wenbiao Yin and Xingzhang Ren and Xinyu Wang and Xinyu Zhang and Xuancheng Ren and Yang Fan and Yang Su and Yichang Zhang and Yinger Zhang and Yu Wan and Yuqiong Liu and Zekun Wang and Zeyu Cui and Zhenru Zhang and Zhipeng Zhou and Zihan Qiu},
    journal = {arXiv preprint arXiv:2505.09388},
    year={2025}
}

@article{qwen2.5,
    title   = {Qwen2.5 Technical Report}, 
    author  = {An Yang and Baosong Yang and Beichen Zhang and Binyuan Hui and Bo Zheng and Bowen Yu and Chengyuan Li and Dayiheng Liu and Fei Huang and Haoran Wei and Huan Lin and Jian Yang and Jianhong Tu and Jianwei Zhang and Jianxin Yang and Jiaxi Yang and Jingren Zhou and Junyang Lin and Kai Dang and Keming Lu and Keqin Bao and Kexin Yang and Le Yu and Mei Li and Mingfeng Xue and Pei Zhang and Qin Zhu and Rui Men and Runji Lin and Tianhao Li and Tingyu Xia and Xingzhang Ren and Xuancheng Ren and Yang Fan and Yang Su and Yichang Zhang and Yu Wan and Yuqiong Liu and Zeyu Cui and Zhenru Zhang and Zihan Qiu},
    journal = {arXiv preprint arXiv:2412.15115},
    year    = {2024}
}

@article{qwen2,
    title   = {Qwen2 Technical Report}, 
    author  = {An Yang and Baosong Yang and Binyuan Hui and Bo Zheng and Bowen Yu and Chang Zhou and Chengpeng Li and Chengyuan Li and Dayiheng Liu and Fei Huang and Guanting Dong and Haoran Wei and Huan Lin and Jialong Tang and Jialin Wang and Jian Yang and Jianhong Tu and Jianwei Zhang and Jianxin Ma and Jin Xu and Jingren Zhou and Jinze Bai and Jinzheng He and Junyang Lin and Kai Dang and Keming Lu and Keqin Chen and Kexin Yang and Mei Li and Mingfeng Xue and Na Ni and Pei Zhang and Peng Wang and Ru Peng and Rui Men and Ruize Gao and Runji Lin and Shijie Wang and Shuai Bai and Sinan Tan and Tianhang Zhu and Tianhao Li and Tianyu Liu and Wenbin Ge and Xiaodong Deng and Xiaohuan Zhou and Xingzhang Ren and Xinyu Zhang and Xipin Wei and Xuancheng Ren and Yang Fan and Yang Yao and Yichang Zhang and Yu Wan and Yunfei Chu and Yuqiong Liu and Zeyu Cui and Zhenru Zhang and Zhihao Fan},
    journal = {arXiv preprint arXiv:2407.10671},
    year    = {2024}
}

@article{eldan2023tinystories,
  title={Tinystories: How small can language models be and still speak coherent english?},
  author={Eldan, Ronen and Li, Yuanzhi},
  journal={arXiv preprint arXiv:2305.07759},
  year={2023}
}

@article{austin2021structured,
  title={Structured denoising diffusion models in discrete state-spaces},
  author={Austin, Jacob and Johnson, Daniel D and Ho, Jonathan and Tarlow, Daniel and Van Den Berg, Rianne},
  journal={Advances in neural information processing systems},
  volume={34},
  pages={17981--17993},
  year={2021}
}

@inproceedings{rombach2022high,
  title={High-resolution image synthesis with latent diffusion models},
  author={Rombach, Robin and Blattmann, Andreas and Lorenz, Dominik and Esser, Patrick and Ommer, Bj{\"o}rn},
  booktitle={Proceedings of the IEEE/CVF conference on computer vision and pattern recognition},
  pages={10684--10695},
  year={2022}
}

@article{ho2020denoising,
  title={Denoising diffusion probabilistic models},
  author={Ho, Jonathan and Jain, Ajay and Abbeel, Pieter},
  journal={Advances in neural information processing systems},
  volume={33},
  pages={6840--6851},
  year={2020}
}

@article{lou2023discrete,
  title={Discrete diffusion modeling by estimating the ratios of the data distribution},
  author={Lou, Aaron and Meng, Chenlin and Ermon, Stefano},
  journal={arXiv preprint arXiv:2310.16834},
  year={2023}
}

@article{wu2025fast1,
  title={Fast-dllm: Training-free acceleration of diffusion llm by enabling kv cache and parallel decoding},
  author={Wu, Chengyue and Zhang, Hao and Xue, Shuchen and Liu, Zhijian and Diao, Shizhe and Zhu, Ligeng and Luo, Ping and Han, Song and Xie, Enze},
  journal={arXiv preprint arXiv:2505.22618},
  year={2025}
}

@article{ben2025accelerated,
  title={Accelerated sampling from masked diffusion models via entropy bounded unmasking},
  author={Ben-Hamu, Heli and Gat, Itai and Severo, Daniel and Nolte, Niklas and Karrer, Brian},
  journal={arXiv preprint arXiv:2505.24857},
  year={2025}
}

@article{kim2025train,
  title={Train for the worst, plan for the best: Understanding token ordering in masked diffusions},
  author={Kim, Jaeyeon and Shah, Kulin and Kontonis, Vasilis and Kakade, Sham and Chen, Sitan},
  journal={arXiv preprint arXiv:2502.06768},
  year={2025}
}

@article{nie2024scaling,
  title={Scaling up masked diffusion models on text},
  author={Nie, Shen and Zhu, Fengqi and Du, Chao and Pang, Tianyu and Liu, Qian and Zeng, Guangtao and Lin, Min and Li, Chongxuan},
  journal={arXiv preprint arXiv:2410.18514},
  year={2024}
}

@article{feng2025theoretical,
  title={Theoretical benefit and limitation of diffusion language model},
  author={Feng, Guhao and Geng, Yihan and Guan, Jian and Wu, Wei and Wang, Liwei and He, Di},
  journal={arXiv preprint arXiv:2502.09622},
  year={2025}
}

@article{li2025diffusion,
  title={Diffusion language models know the answer before decoding},
  author={Li, Pengxiang and Zhou, Yefan and Muhtar, Dilxat and Yin, Lu and Yan, Shilin and Shen, Li and Vosoughi, Soroush and Liu, Shiwei},
  journal={arXiv preprint arXiv:2508.19982},
  year={2025}
}

@article{vaina2023diffusion,
  title={Diffusion Language Models Generation Can Be Halted Early},
  author={Vaina, Sofia Maria Lo Cicero and Balagansky, Nikita and Gavrilov, Daniil},
  journal={arXiv preprint arXiv:2305.10818},
  year={2023}
}

@article{goel2026skip,
  title={Skip to the Good Part: Representation Structure \& Inference-Time Layer Skipping in Diffusion vs. Autoregressive LLMs},
  author={Goel, Raghavv and Garrepalli, Risheek and Agrawal, Sudhanshu and Lott, Chris and Lee, Mingu and Porikli, Fatih},
  journal={arXiv preprint arXiv:2603.07475},
  year={2026}
}

@article{ni2025diffusion,
  title={Diffusion language models are super data learners},
  author={Ni, Jinjie and Liu, Qian and Dou, Longxu and Du, Chao and Wang, Zili and Yan, Hang and Pang, Tianyu and Shieh, Michael Qizhe},
  journal={arXiv preprint arXiv:2511.03276},
  year={2025}
}

@article{gao2025makes,
  title={What Makes Diffusion Language Models Super Data Learners?},
  author={Gao, Zitian and Luo, Haoming and Chen, Lynx and Liu, Jason Klein and Tao, Ran and Zhou, Joey and Dai, Bryan},
  journal={arXiv preprint arXiv:2510.04071},
  year={2025}
}

@article{prabhudesai2025diffusion,
  title={Diffusion beats autoregressive in data-constrained settings},
  author={Prabhudesai, Mihir and Wu, Mengning and Zadeh, Amir and Fragkiadaki, Katerina and Pathak, Deepak},
  journal={arXiv preprint arXiv:2507.15857},
  year={2025}
}

@article{shin2026understanding,
  title={Understanding the Reversal Curse Mitigation in Masked Diffusion Models through Attention and Training Dynamics},
  author={Shin, Sangwoo and Kim, BumJun and Lee, Kyelim and Jeon, Moongyu and No, Albert},
  journal={arXiv preprint arXiv:2602.02133},
  year={2026}
}

@article{yang2025powerful,
  title={On Powerful Ways to Generate: Autoregression, Diffusion, and Beyond},
  author={Yang, Chenxiao and Zhou, Cai and Wipf, David and Li, Zhiyuan},
  journal={arXiv preprint arXiv:2510.06190},
  year={2025}
}

@article{ni2026flexibility,
  title={The Flexibility Trap: Why Arbitrary Order Limits Reasoning Potential in Diffusion Language Models},
  author={Ni, Zanlin and Wang, Shenzhi and Yue, Yang and Yu, Tianyu and Zhao, Weilin and Hua, Yeguo and Chen, Tianyi and Song, Jun and Yu, Cheng and Zheng, Bo and others},
  journal={arXiv preprint arXiv:2601.15165},
  year={2026}
}

@article{li2025breaking,
  title={Breaking AR's Sampling Bottleneck: Provable Acceleration via Diffusion Language Models},
  author={Li, Gen and Cai, Changxiao},
  journal={arXiv preprint arXiv:2505.21400},
  year={2025}
}

@article{jiang2025diffusion,
  title={Diffusion Language Models are Provably Optimal Parallel Samplers},
  author={Jiang, Haozhe and Haghtalab, Nika and Chen, Lijie},
  journal={arXiv preprint arXiv:2512.25014},
  year={2025}
}

@article{wang2026analyzing,
  title={Analyzing Diffusion and Autoregressive Vision Language Models in Multimodal Embedding Space},
  author={Wang, Zihang and Zhang, Siyue and Zhao, Yilun and Yang, Jingyi and Song, Tingyu and Luu, Anh Tuan and Zhao, Chen},
  journal={arXiv preprint arXiv:2602.06056},
  year={2026}
}

@inproceedings{zhang2025diffusion,
  title={Diffusion vs. autoregressive language models: A text embedding perspective},
  author={Zhang, Siyue and Zhao, Yilun and Geng, Liyuan and Cohan, Arman and Tuan, Luu Anh and Zhao, Chen},
  booktitle={Proceedings of the 2025 Conference on Empirical Methods in Natural Language Processing},
  pages={4273--4303},
  year={2025}
}

@article{rahimi2026step,
  title={Step-Wise Refusal Dynamics in Autoregressive and Diffusion Language Models},
  author={Rahimi, Eliron and Hirshel, Elad and Himelstein, Rom and LeVi, Amit and Mendelson, Avi and Baskin, Chaim},
  journal={arXiv preprint arXiv:2602.02600},
  year={2026}
}

@article{wen2025devil,
  title={The devil behind the mask: An emergent safety vulnerability of diffusion llms},
  author={Wen, Zichen and Qu, Jiashu and Chen, Zhaorun and Lu, Xiaoya and Liu, Dongrui and Liu, Zhiyuan and Wu, Ruixi and Yang, Yicun and Jin, Xiangqi and Xu, Haoyun and others},
  journal={arXiv preprint arXiv:2507.11097},
  year={2025}
}

@article{zhang2025jailbreaking,
  title={Jailbreaking large language diffusion models: Revealing hidden safety flaws in diffusion-based text generation},
  author={Zhang, Yuanhe and Xie, Fangzhou and Zhou, Zhenhong and Li, Zherui and Chen, Hao and Wang, Kun and Guo, Yufei},
  journal={arXiv preprint arXiv:2507.19227},
  year={2025}
}

@inproceedings{dikov2025diffusion,
  title={Diffusion vs Autoregression: An Empirical Study on Code Comment Translation},
  author={Dikov, Alexander and Zvorygin, Vladimir and Malykh, Valentin},
  booktitle={2025 5th International Conference on Code Quality (ICCQ)},
  pages={55--63},
  year={2025},
  organization={IEEE}
}

@article{kitouni2024factorization,
  title={The factorization curse: Which tokens you predict underlie the reversal curse and more},
  author={Kitouni, Ouail and Nolte, Niklas and Bouchacourt, Diane and Williams, Adina and Rabbat, Mike and Ibrahim, Mark},
  journal={Advances in Neural Information Processing Systems},
  volume={37},
  pages={112329--112355},
  year={2024}
}

@article{xiong2025unveiling,
  title={Unveiling the potential of diffusion large language model in controllable generation},
  author={Xiong, Zhen and Cai, Yujun and Li, Zhecheng and Wang, Yiwei},
  journal={arXiv preprint arXiv:2507.04504},
  year={2025}
}

@article{song2025seed,
  title={Seed diffusion: A large-scale diffusion language model with high-speed inference},
  author={Song, Yuxuan and Zhang, Zheng and Luo, Cheng and Gao, Pengyang and Xia, Fan and Luo, Hao and Li, Zheng and Yang, Yuehang and Yu, Hongli and Qu, Xingwei and others},
  journal={arXiv preprint arXiv:2508.02193},
  year={2025}
}

@article{khanna2025mercury,
  title={Mercury: Ultra-fast language models based on diffusion},
  author={Khanna, Samar and Kharbanda, Siddhant and Li, Shufan and Varma, Harshit and Wang, Eric and Birnbaum, Sawyer and Luo, Ziyang and Miraoui, Yanis and Palrecha, Akash and Ermon, Stefano and others},
  journal={arXiv e-prints},
  pages={arXiv--2506},
  year={2025}
}

@article{zhu2025llada,
  title={LLaDA-MoE: A Sparse MoE Diffusion Language Model},
  author={Fengqi Zhu and Zebin You and Yipeng Xing and Zenan Huang and Lin Liu and Yihong Zhuang and Guoshan Lu and Kangyu Wang and Xudong Wang and Lanning Wei and Hongrui Guo and Jiaqi Hu and Wentao Ye and Tieyuan Chen and Chenchen Li and Chengfu Tang and Haibo Feng and Jun Hu and Jun Zhou and Xiaolu Zhang and Zhenzhong Lan and Junbo Zhao and Da Zheng and Chongxuan Li and Jianguo Li and Ji-Rong Wen},
  journal={arXiv preprint arXiv:2509.24389},
  year={2025}
}

@article{shi2024simplified,
  title={Simplified and generalized masked diffusion for discrete data},
  author={Shi, Jiaxin and Han, Kehang and Wang, Zhe and Doucet, Arnaud and Titsias, Michalis},
  journal={Advances in neural information processing systems},
  volume={37},
  pages={103131--103167},
  year={2024}
}

@article{ou2024your,
  title={Your absorbing discrete diffusion secretly models the conditional distributions of clean data},
  author={Ou, Jingyang and Nie, Shen and Xue, Kaiwen and Zhu, Fengqi and Sun, Jiacheng and Li, Zhenguo and Li, Chongxuan},
  journal={arXiv preprint arXiv:2406.03736},
  year={2024}
}

@article{sahoo2024simple,
  title={Simple and effective masked diffusion language models},
  author={Sahoo, Subham S and Arriola, Marianne and Schiff, Yair and Gokaslan, Aaron and Marroquin, Edgar and Chiu, Justin T and Rush, Alexander and Kuleshov, Volodymyr},
  journal={Advances in Neural Information Processing Systems},
  volume={37},
  pages={130136--130184},
  year={2024}
}

@article{parola2023speech,
  title={Speech disturbances in schizophrenia: Assessing cross-linguistic generalizability of NLP automated measures of coherence},
  author={Parola, Alberto and Lin, Jessica Mary and Simonsen, Arndis and Bliksted, Vibeke and Zhou, Yuan and Wang, Huiling and Inoue, Lana and Koelkebeck, Katja and Fusaroli, Riccardo},
  journal={Schizophrenia Research},
  volume={259},
  pages={59--70},
  year={2023},
  publisher={Elsevier}
}

@article{bedi2015automated,
  title={Automated analysis of free speech predicts psychosis onset in high-risk youths},
  author={Bedi, Gillinder and Carrillo, Facundo and Cecchi, Guillermo A and Slezak, Diego Fern{\'a}ndez and Sigman, Mariano and Mota, Nat{\'a}lia B and Ribeiro, Sidarta and Javitt, Daniel C and Copelli, Mauro and Corcoran, Cheryl M},
  journal={npj Schizophrenia},
  volume={1},
  number={1},
  pages={1--7},
  year={2015},
  publisher={Nature Publishing Group}
}

@article{kirk2023understanding,
  title={Understanding the effects of rlhf on llm generalisation and diversity},
  author={Kirk, Robert and Mediratta, Ishita and Nalmpantis, Christoforos and Luketina, Jelena and Hambro, Eric and Grefenstette, Edward and Raileanu, Roberta},
  journal={arXiv preprint arXiv:2310.06452},
  year={2023}
}

@article{marshall1979inequalities,
  title={Inequalities: theory of majorization and its applications},
  author={Marshall, Albert W and Olkin, Ingram and Arnold, Barry C},
  year={1979},
  publisher={Springer}
}

@book{bremaud2012introduction,
  title={An introduction to probabilistic modeling},
  author={Br{\'e}maud, Pierre},
  year={2012},
  publisher={Springer Science \& Business Media}
}

@article{arriola2025block,
  title={Block diffusion: Interpolating between autoregressive and diffusion language models},
  author={Arriola, Marianne and Gokaslan, Aaron and Chiu, Justin T and Yang, Zhihan and Qi, Zhixuan and Han, Jiaqi and Sahoo, Subham Sekhar and Kuleshov, Volodymyr},
  journal={arXiv preprint arXiv:2503.09573},
  year={2025}
}

@article{guo2025deepseek,
  title={Deepseek-r1: Incentivizing reasoning capability in llms via reinforcement learning},
  author={Guo, Daya and Yang, Dejian and Zhang, Haowei and Song, Junxiao and Wang, Peiyi and Zhu, Qihao and Xu, Runxin and Zhang, Ruoyu and Ma, Shirong and Bi, Xiao and others},
  journal={arXiv preprint arXiv:2501.12948},
  year={2025}
}

@misc{gemmateam2025gemma3technicalreport,
      title={Gemma 3 Technical Report}, 
      author={Gemma Team and Aishwarya Kamath and Johan Ferret and Shreya Pathak and Nino Vieillard and Ramona Merhej and Sarah Perrin and Tatiana Matejovicova and Alexandre Ramé and Morgane Rivière and Louis Rouillard and Thomas Mesnard and Geoffrey Cideron and Jean-bastien Grill and Sabela Ramos and Edouard Yvinec and Michelle Casbon and Etienne Pot and Ivo Penchev and Gaël Liu and Francesco Visin and Kathleen Kenealy and Lucas Beyer and Xiaohai Zhai and Anton Tsitsulin and Robert Busa-Fekete and Alex Feng and Noveen Sachdeva and Benjamin Coleman and Yi Gao and Basil Mustafa and Iain Barr and Emilio Parisotto and David Tian and Matan Eyal and Colin Cherry and Jan-Thorsten Peter and Danila Sinopalnikov and Surya Bhupatiraju and Rishabh Agarwal and Mehran Kazemi and Dan Malkin and Ravin Kumar and David Vilar and Idan Brusilovsky and Jiaming Luo and Andreas Steiner and Abe Friesen and Abhanshu Sharma and Abheesht Sharma and Adi Mayrav Gilady and Adrian Goedeckemeyer and Alaa Saade and Alex Feng and Alexander Kolesnikov and Alexei Bendebury and Alvin Abdagic and Amit Vadi and András György and André Susano Pinto and Anil Das and Ankur Bapna and Antoine Miech and Antoine Yang and Antonia Paterson and Ashish Shenoy and Ayan Chakrabarti and Bilal Piot and Bo Wu and Bobak Shahriari and Bryce Petrini and Charlie Chen and Charline Le Lan and Christopher A. Choquette-Choo and CJ Carey and Cormac Brick and Daniel Deutsch and Danielle Eisenbud and Dee Cattle and Derek Cheng and Dimitris Paparas and Divyashree Shivakumar Sreepathihalli and Doug Reid and Dustin Tran and Dustin Zelle and Eric Noland and Erwin Huizenga and Eugene Kharitonov and Frederick Liu and Gagik Amirkhanyan and Glenn Cameron and Hadi Hashemi and Hanna Klimczak-Plucińska and Harman Singh and Harsh Mehta and Harshal Tushar Lehri and Hussein Hazimeh and Ian Ballantyne and Idan Szpektor and Ivan Nardini and Jean Pouget-Abadie and Jetha Chan and Joe Stanton and John Wieting and Jonathan Lai and Jordi Orbay and Joseph Fernandez and Josh Newlan and Ju-yeong Ji and Jyotinder Singh and Kat Black and Kathy Yu and Kevin Hui and Kiran Vodrahalli and Klaus Greff and Linhai Qiu and Marcella Valentine and Marina Coelho and Marvin Ritter and Matt Hoffman and Matthew Watson and Mayank Chaturvedi and Michael Moynihan and Min Ma and Nabila Babar and Natasha Noy and Nathan Byrd and Nick Roy and Nikola Momchev and Nilay Chauhan and Noveen Sachdeva and Oskar Bunyan and Pankil Botarda and Paul Caron and Paul Kishan Rubenstein and Phil Culliton and Philipp Schmid and Pier Giuseppe Sessa and Pingmei Xu and Piotr Stanczyk and Pouya Tafti and Rakesh Shivanna and Renjie Wu and Renke Pan and Reza Rokni and Rob Willoughby and Rohith Vallu and Ryan Mullins and Sammy Jerome and Sara Smoot and Sertan Girgin and Shariq Iqbal and Shashir Reddy and Shruti Sheth and Siim Põder and Sijal Bhatnagar and Sindhu Raghuram Panyam and Sivan Eiger and Susan Zhang and Tianqi Liu and Trevor Yacovone and Tyler Liechty and Uday Kalra and Utku Evci and Vedant Misra and Vincent Roseberry and Vlad Feinberg and Vlad Kolesnikov and Woohyun Han and Woosuk Kwon and Xi Chen and Yinlam Chow and Yuvein Zhu and Zichuan Wei and Zoltan Egyed and Victor Cotruta and Minh Giang and Phoebe Kirk and Anand Rao and Kat Black and Nabila Babar and Jessica Lo and Erica Moreira and Luiz Gustavo Martins and Omar Sanseviero and Lucas Gonzalez and Zach Gleicher and Tris Warkentin and Vahab Mirrokni and Evan Senter and Eli Collins and Joelle Barral and Zoubin Ghahramani and Raia Hadsell and Yossi Matias and D. Sculley and Slav Petrov and Noah Fiedel and Noam Shazeer and Oriol Vinyals and Jeff Dean and Demis Hassabis and Koray Kavukcuoglu and Clement Farabet and Elena Buchatskaya and Jean-Baptiste Alayrac and Rohan Anil and Dmitry and Lepikhin and Sebastian Borgeaud and Olivier Bachem and Armand Joulin and Alek Andreev and Cassidy Hardin and Robert Dadashi and Léonard Hussenot},
      year={2025},
      eprint={2503.19786},
      archivePrefix={arXiv},
      primaryClass={cs.CL},
      url={https://arxiv.org/abs/2503.19786}, 
}

@article{abouelenin2025phi,
  title={Phi-4-mini technical report: Compact yet powerful multimodal language models via mixture-of-loras},
  author={Abouelenin, Abdelrahman and Ashfaq, Atabak and Atkinson, Adam and Awadalla, Hany and Bach, Nguyen and Bao, Jianmin and Benhaim, Alon and Cai, Martin and Chaudhary, Vishrav and Chen, Congcong and others},
  journal={arXiv preprint arXiv:2503.01743},
  year={2025}
}

@article{yang2025qwen3,
  title={Qwen3 technical report},
  author={Yang, An and Li, Anfeng and Yang, Baosong and Zhang, Beichen and Hui, Binyuan and Zheng, Bo and Yu, Bowen and Gao, Chang and Huang, Chengen and Lv, Chenxu and others},
  journal={arXiv preprint arXiv:2505.09388},
  year={2025}
}

@article{xie2025dream,
  title={Dream-coder 7b: An open diffusion language model for code},
  author={Xie, Zhihui and Ye, Jiacheng and Zheng, Lin and Gao, Jiahui and Dong, Jingwei and Wu, Zirui and Zhao, Xueliang and Gong, Shansan and Jiang, Xin and Li, Zhenguo and others},
  journal={arXiv preprint arXiv:2509.01142},
  year={2025}
}

@article{grattafiori2024llama,
  title={The llama 3 herd of models},
  author={Grattafiori, Aaron and Dubey, Abhimanyu and Jauhri, Abhinav and Pandey, Abhinav and Kadian, Abhishek and Al-Dahle, Ahmad and Letman, Aiesha and Mathur, Akhil and Schelten, Alan and Vaughan, Alex and others},
  journal={arXiv preprint arXiv:2407.21783},
  year={2024}
}

@article{Jiang2023Mistral7,
  title={Mistral 7B},
  author={Albert Qiaochu Jiang and Alexandre Sablayrolles and Arthur Mensch and Chris Bamford and Devendra Singh Chaplot and Diego de Las Casas and Florian Bressand and Gianna Lengyel and Guillaume Lample and Lucile Saulnier and L{\'e}lio Renard Lavaud and Marie-Anne Lachaux and Pierre Stock and Teven Le Scao and Thibaut Lavril and Thomas Wang and Timoth{\'e}e Lacroix and William El Sayed},
  journal={ArXiv},
  year={2023},
  volume={abs/2310.06825},
  url={https://api.semanticscholar.org/CorpusID:263830494}
}

@article{touvron2023llama2,
  title={Llama 2: Open foundation and fine-tuned chat models},
  author={Touvron, Hugo and Martin, Louis and Stone, Kevin and Albert, Peter and Almahairi, Amjad and Babaei, Yasmine and Bashlykov, Nikolay and Batra, Soumya and Bhargava, Prajjwal and Bhosale, Shruti and others},
  journal={arXiv preprint arXiv:2307.09288},
  year={2023}
}
\bibliographystyle{colm2026_conference}

\clearpage

\appendix

\section{Evaluation details of off-the-shelf models}
\label{app:offtheshelf}

We collect a dataset of texts generated by off-the-shelf DLMs and ARMs for evaluation.

\textbf{Generation configuration.} We randomly sample 1000 examples from the FineWeb dataset~\citep{fineweb}. For each example, we use its first 30 tokens as a prompt and generate 20 continuations using 10 ARMs and 10 DLMs, with a maximum length of 512 tokens for each continuation. \Tref{tab:datacollection} lists the employed models.

We use the prompt template ``\texttt{Continue the following text in approximately 500 words: [PROMPT]}'', along with model-specific chat templates. For generation hyperparameters, we adopt the default settings for each model, and we disable thinking mode for all models. These models all have different architectures and training procedures, so adopting the default generation setting is a fair basis for comparison since it demonstrates how each model is ``intended'' to be used.

\textbf{Data filtering and preprocessing.} Since degenerate data severely interfere with our analysis, we filter out degenerate samples with extremely short length or excessively repeated content. The number of samples remaining after filtering is reported in \Tref{tab:datacollection}.

\begin{table}[htb]
\begin{center}
\begin{tabular}{clc}
\toprule
\multicolumn{1}{c}{\bf Type} & \multicolumn{1}{l}{\bf Model} & \multicolumn{1}{c}{\bf Count}\\
\midrule
\multirow{10}{*}{ARM} & DeepSeek-R1-0528-Qwen3-8B~\citep{guo2025deepseek} & 946 \\
& Gemma-3-4B-it~\citep{gemmateam2025gemma3technicalreport} & 996 \\
& Llama-2-7B-chat~\citep{touvron2023llama2} & 997 \\
& Llama-3.1-8B-Instruct~\citep{grattafiori2024llama} & 999 \\
& Llama-3.2-3B-Instruct~\citep{grattafiori2024llama}  & 996 \\
& Mistral-7B-Instruct-v0.3~\citep{Jiang2023Mistral7} & 995 \\
& Phi-4-mini-Instruct~\citep{abouelenin2025phi} & 996 \\
& Qwen2-7B-Instruct~\citep{qwen2} & 993 \\
& Qwen2.5-7B-Instruct~\citep{qwen2.5} & 991 \\
& Qwen3-8B~\citep{yang2025qwen3} & 993 \\
\midrule
\multirow{10}{*}{DLM} & Dream-Coder-v0-Instruct-7B~\citep{xie2025dream} & 694 \\
& Dream-v0-Instruct-7B~\citep{ye2025dream} & 215 \\
& Fast-dLLM-v2-7B~\citep{wu2025fast} & 982 \\
& LLaDA-8B-Instruct~\citep{nie2025large} & 902 \\
& LLaDA-MoE-7B-A1B-Instruct~\citep{zhu2025llada} & 731 \\
& SDAR-4B-Chat~\citep{cheng2025sdar} & 854 \\
& SDAR-4B-Chat-b8~\citep{cheng2025sdar} & 861 \\
& SDAR-8B-Chat~\citep{cheng2025sdar} & 871 \\
& SDLM-3B-D4~\citep{liu2025sdlm} & 652 \\
& SDLM-3B-D8~\citep{liu2025sdlm} & 689 \\
\bottomrule
\end{tabular}
\end{center}
\caption{List of evaluated models and the number of remaining samples for each model.}
\label{tab:datacollection}
\end{table}

\textbf{Evaluation.} The per-model evaluation results are shown in \Tref{tab:offtheshelf}, and a scatter plot visualization of the same results is provided in \Fref{fig:offtheshelf_scatter}. For $n$-gram entropy, we use the Qwen3 tokenizer~\citep{qwen3} for tokenization. For semantic coherence and semantic diversity, we use \texttt{bge-m3}~\citep{multi2024m3} as the embedding model. Off-the-shelf DLMs on average exhibit lower $n$-gram entropy alongside higher semantic coherence and higher semantic diversity compared to ARMs.

\begin{table}[htb]
\begin{center}
\begin{tabular}{clccccc}
\toprule
\multirow{2.5}{*}{\bf Type} &\multirow{2.5}{*}{\bf Model} &\multicolumn{3}{c}{\bf $n$-gram entropy}&\multirow{2.5}{*}{\bf \makecell[c]{Semantic\\coherence}}&\multirow{2.5}{*}{\bf \makecell[c]{Semantic\\diversity}} \\
\cmidrule(lr){3-5}
& & Unigram & Bigram & Trigram \\
\midrule
\multirow{11.5}{*}{ARM} & DeepSeek-R1-0528-Qwen3-8B & 7.579 & 11.385 & 12.655 & 0.575 & 0.569 \\
& Gemma-3-4B-it & 7.411 & 11.241 & 12.639 & 0.557 & 0.572 \\
& Llama-2-7B-chat & 7.013 & 10.612 & 12.027 & 0.634 & 0.604 \\
& Llama-3.1-8B-Instruct & 7.129 & 10.816 & 12.316 & 0.626 & 0.604 \\
& Llama-3.2-3B-Instruct & 7.059 & 10.723 & 12.271 & 0.632 & 0.599 \\
& Mistral-7B-Instruct-v0.3 & 7.237 & 10.957 & 12.428 & 0.616 & 0.603 \\
& Phi-4-mini-Instruct & 7.592 & 11.477 & 12.725 & 0.597 & 0.587 \\
& Qwen2-7B-Instruct & 7.497 & 11.395 & 12.702 & 0.602 & 0.581 \\
& Qwen2.5-7B-Instruct & 7.464 & 11.205 & 12.641 & 0.604 & 0.590 \\
& Qwen3-8B & 7.185 & 11.046 & 12.560 & 0.601 & 0.583 \\
\cmidrule(lr){2-7}
& \bf Avg. & \bf 7.316 & \bf 11.086 & \bf 12.497 & \bf 0.604 & \bf 0.589 \\
\midrule
\multirow{11.5}{*}{DLM} & Dream-Coder-v0-Instruct-7B & 6.777 & 10.030 & 11.262 & 0.625 & 0.592 \\
& Dream-v0-Instruct-7B & 6.541 & \phantom{0}9.134 & \phantom{0}9.912 & 0.654 & 0.620 \\
& Fast-dLLM-v2-7B & 6.838 & 10.370 & 11.896 & 0.633 & 0.590 \\
& LLaDA-8B-Instruct & 6.853 & 10.292 & 11.713 & 0.661 & 0.594 \\
& LLaDA-MoE-7B-A1B-Instruct & 6.752 & 10.098 & 11.449 & 0.665 & 0.580 \\
& SDAR-4B-Chat & 6.972 & 10.506 & 11.898 & 0.641 & 0.611 \\
& SDAR-4B-Chat-b8 & 6.931 & 10.488 & 11.937 & 0.643 & 0.602 \\
& SDAR-8B-Chat & 6.918 & 10.518 & 11.963 & 0.640 & 0.605 \\
& SDLM-3B-D4 & 6.692 & \phantom{0}9.967 & 11.270 & 0.659 & 0.594 \\
& SDLM-3B-D8 & 6.659 & \phantom{0}9.926 & 11.257 & 0.665 & 0.589 \\
\cmidrule(lr){2-7}
& \bf Avg. & \bf 6.793 & \bf 10.133 & \bf 11.456 & \bf 0.649 & \bf 0.597 \\
\bottomrule
\end{tabular}
\end{center}
\caption{Evaluation results for 20 off-the-shelf ARMs and DLMs. ``Avg.'' denotes the average over models within each type.}
\label{tab:offtheshelf}
\end{table}

\begin{figure}[htb]
    \begin{center}
        \includegraphics[width=\linewidth]{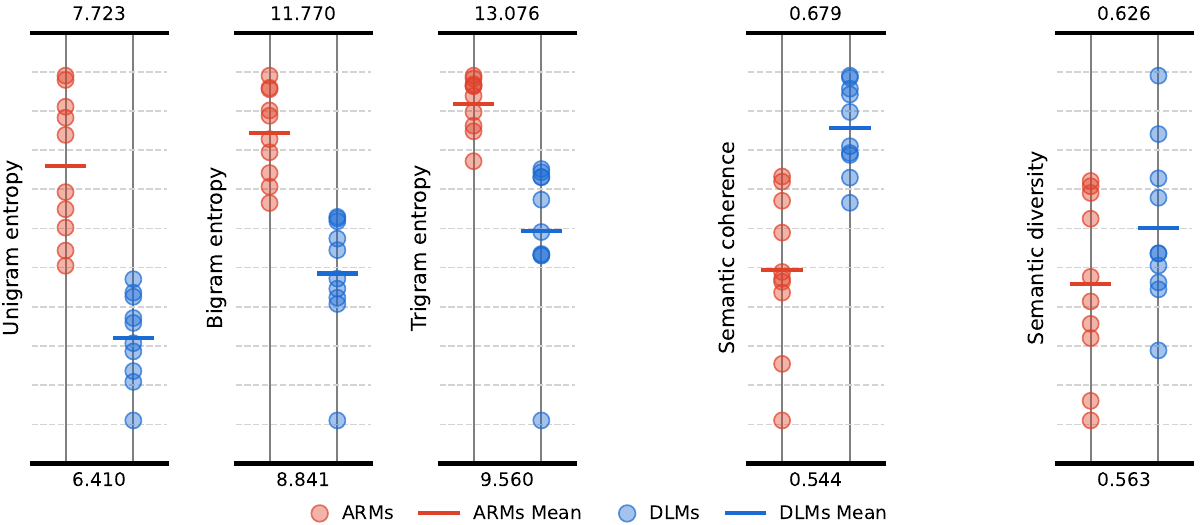}
    \end{center}
    \caption{Scatter plot visualization of the off-the-shelf evaluation results in \Tref{tab:offtheshelf}. Each point denotes one model, and the horizontal bars indicate the mean over models within each type.}
    \label{fig:offtheshelf_scatter}
\end{figure}
\clearpage

\section{Additional details for controlled experiments on training objectives}
\label{app:train}

\subsection{Training objectives selection}
\label{app:trainobj}

We select all possible combinations of the four decomposed components listed in \Sref{sec:decomposition} subject to two rules. First, when bidirectional context is used, input masking and label masking are necessarily included: without them, the model would have access to the full clean sequence at both input and output, reducing the task to identity mapping rather than prediction. Second, the weighting function is only varied when label masking is present, since $\omega_t = 1/t$ serves to normalize the loss magnitude. Specifically, it compensates for the fact that the number of loss terms (masked positions) is proportional to $t$.

\subsection{Experimental setup details}
\label{app:trainsetup}

For the main experiments, we use a 120M-parameter LLaMA architecture~\citep{touvron2023llama}, configured with a hidden size of 768, 12 attention heads, and 12 layers. We optimize all models using AdamW~\citep{loshchilov2017decoupled} optimizer with a global batch size of 1024 and a peak learning rate of 1e-3, which warms up over the first 100 steps and remains constant thereafter. Text generation is performed using standard sequential autoregressive decoding with a temperature of $1.0$, without top-$p$ or top-$k$ sampling. This setup ensures that the generated outputs faithfully reflect the learned distribution, without introducing additional distortion.

As a control variable, we also evaluate the next-token prediction loss at each checkpoint. We compute it using a uniform procedure: sample $x \sim \mathcal{D}_{\text{data}}$; for each $i \in \{2, \dots, L\}$, mask all tokens $x^j$ with $j \geq i$ to obtain $\tilde{x}$; then feed $\tilde{x}$ into the model and compute the negative log likelihood of predicting $x^i$ from $\tilde{x}$. Averaging this quantity over $i$ and then taking expectation over $x$ gives the next-token prediction loss used in evaluation.

\subsection{Robustness concern}
\label{app:trainrobust}

As an additional robustness check, we repeat the main controlled experiments in \Sref{sec:train} with three random seeds. Apart from the main experiments, we also verify the results across different model architectures and datasets. For the architecture ablation, we train another set of models using a 30M-parameter Qwen2 architecture~\citep{qwen2}, configured with a hidden size of 512, 8 attention heads, and 8 layers. Other hyperparameters remain the same as those in \Sref{sec:train}. For the dataset ablation, we train a set of models on the TinyStories dataset~\citep{eldan2023tinystories} for 5 epochs each. Because the texts in the TinyStories dataset are shorter, we reduce the context length to 128 and, accordingly, the learning rate to 5e-4. 

The results for the random seed, architecture, and dataset ablation are shown in \Fref{fig:trainseed},  \Fref{fig:qwentrain}, and \Fref{fig:tinystorytrain}, respectively. All additional experiments exhibit the same phenomena as in \Sref{sec:train}: clustering effects are observed in semantic coherence and semantic diversity, while entropy shows no significant differences, and context scope remains the primary driver. These results further confirm the robustness of our conclusion.

\begin{figure}[htb]
    \centering
    \begin{subfigure}{\linewidth}
        \includegraphics[width=\linewidth]{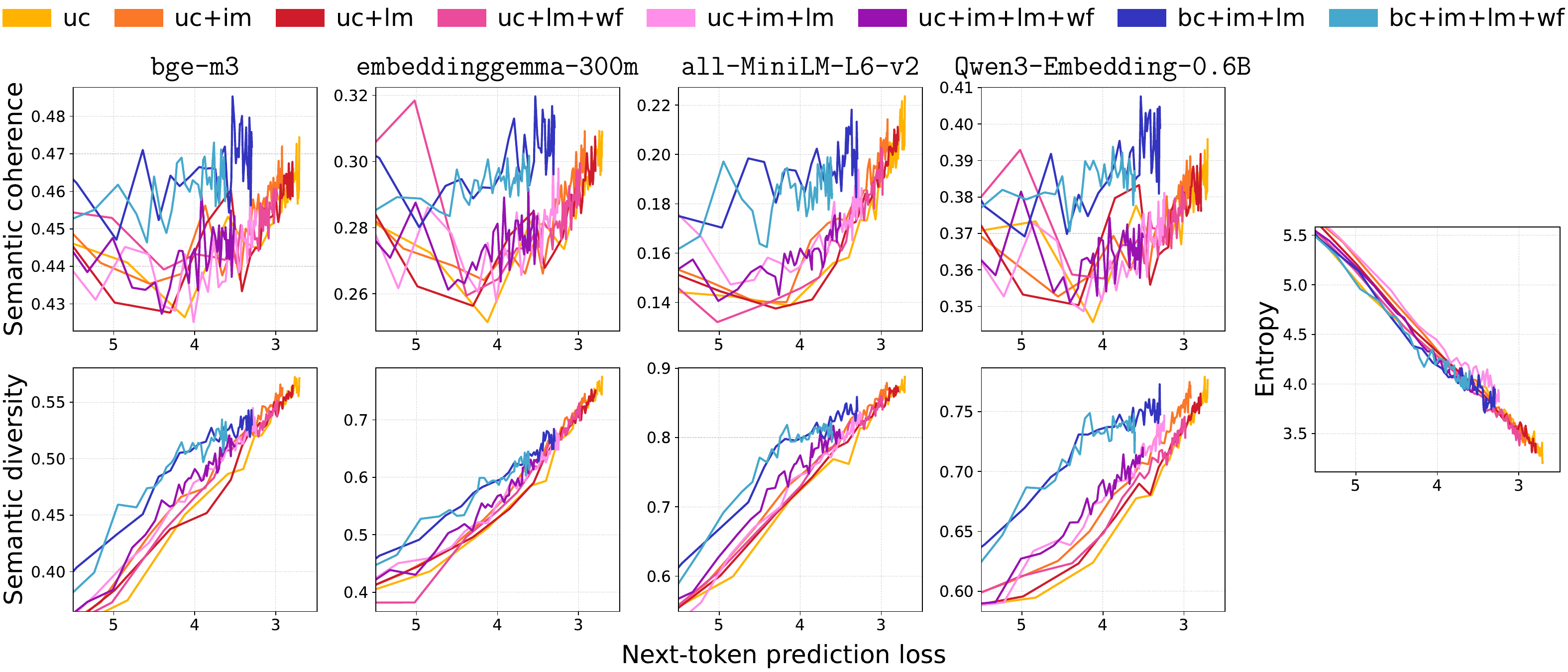}
        \caption{seed = 0}
        \label{fig:trainseed_0}
    \end{subfigure}
    
    \vspace{1em}
    
    \begin{subfigure}{\linewidth}
        \includegraphics[width=\linewidth]{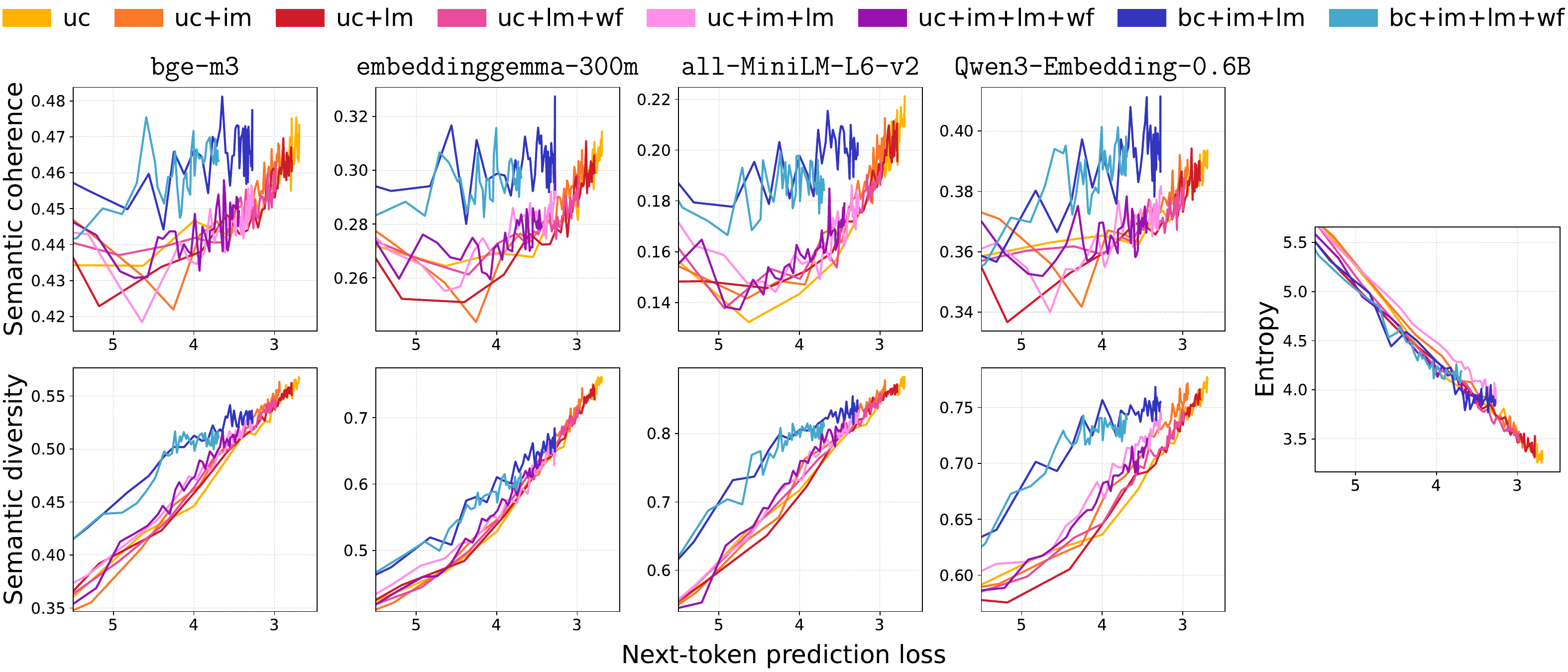}
        \caption{seed = 1}
        \label{fig:trainseed_1}
    \end{subfigure}
    
    \vspace{1em}
    
    \begin{subfigure}{\linewidth}
        \includegraphics[width=\linewidth]{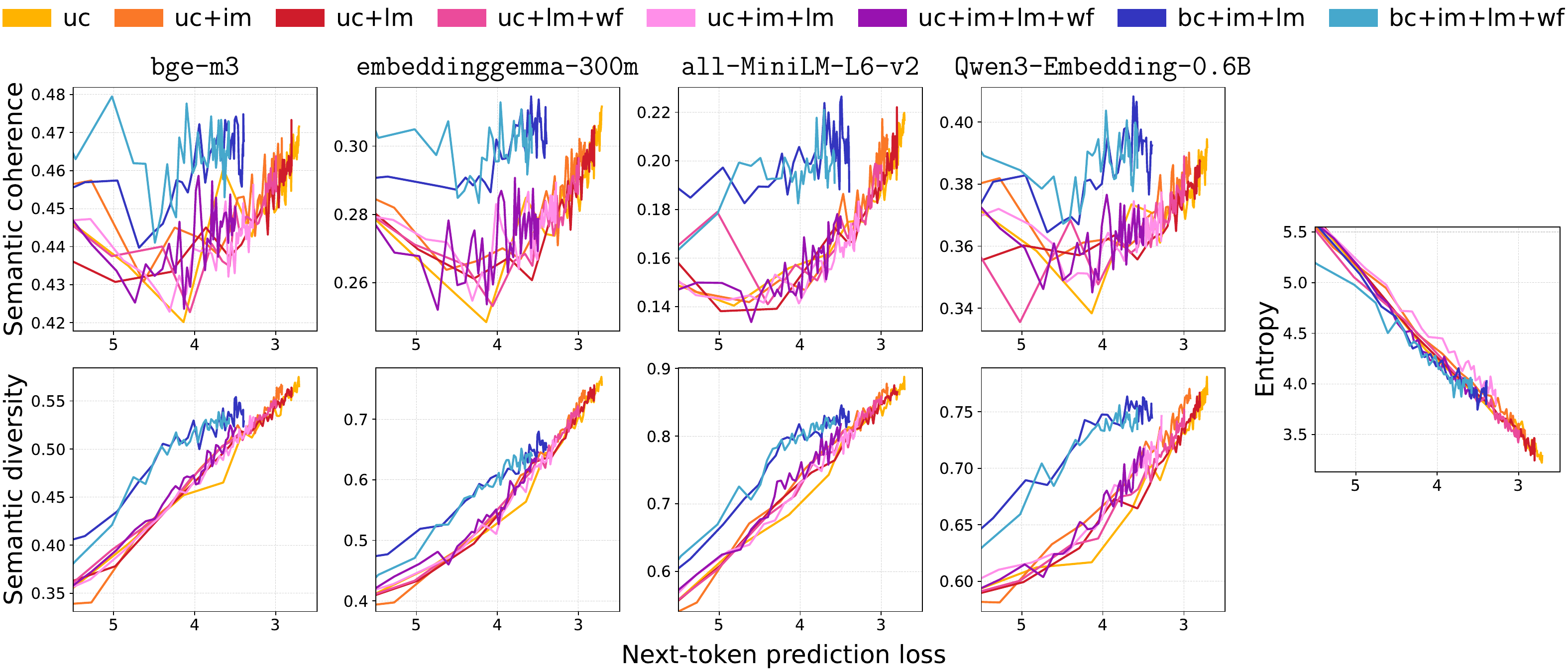}
        \caption{seed = 2}
        \label{fig:trainseed_2}
    \end{subfigure}
    
    \caption{Controlled experiments on interpolated training objectives with three different random seeds, presented in the same format as \Fref{fig:trainingresults}.}
    \label{fig:trainseed}
\end{figure}

\begin{figure}[htb]
    \begin{center}
    \includegraphics[width=\linewidth]{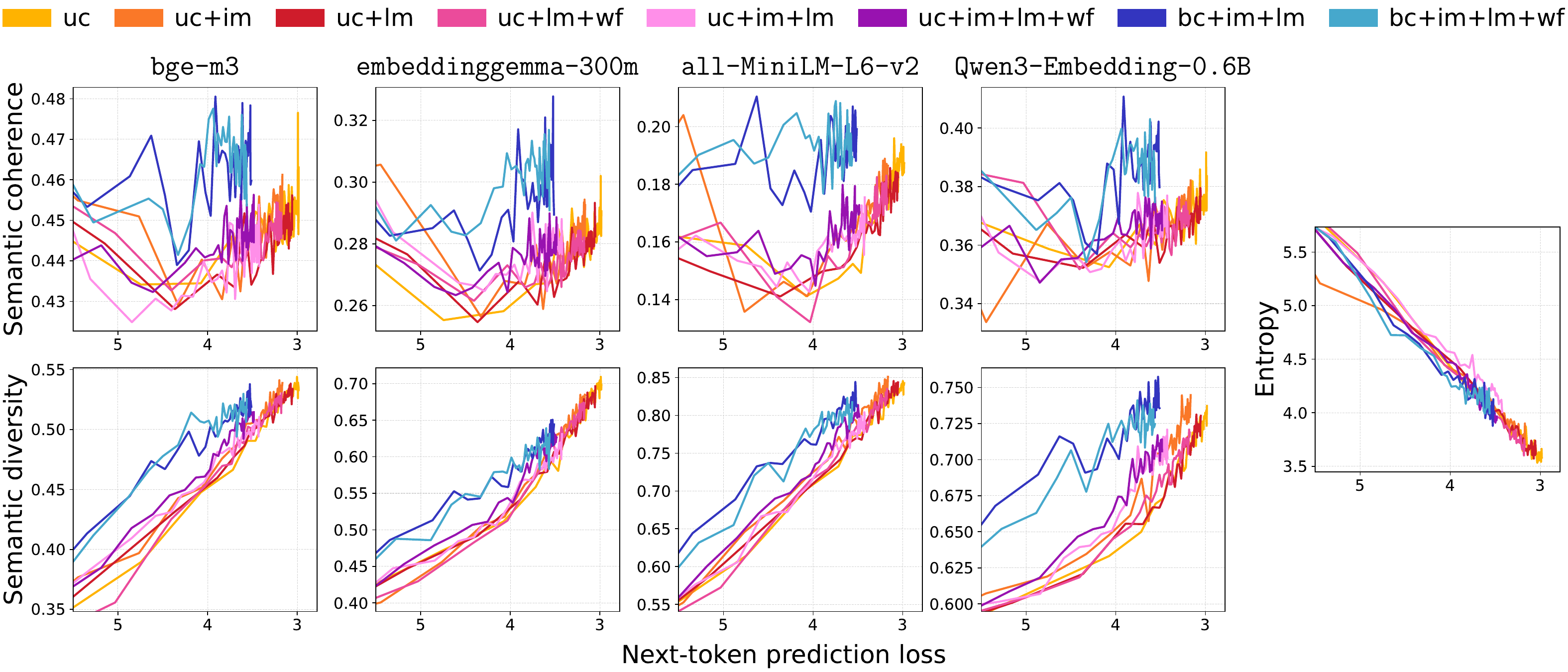}
    \end{center}
    \caption{Controlled experiments on interpolated training objectives using the Qwen2 architecture, presented in the same format as \Fref{fig:trainingresults}.}
    \label{fig:qwentrain}
\end{figure}

\begin{figure}[htb]
    \begin{center}
    \includegraphics[width=\linewidth]{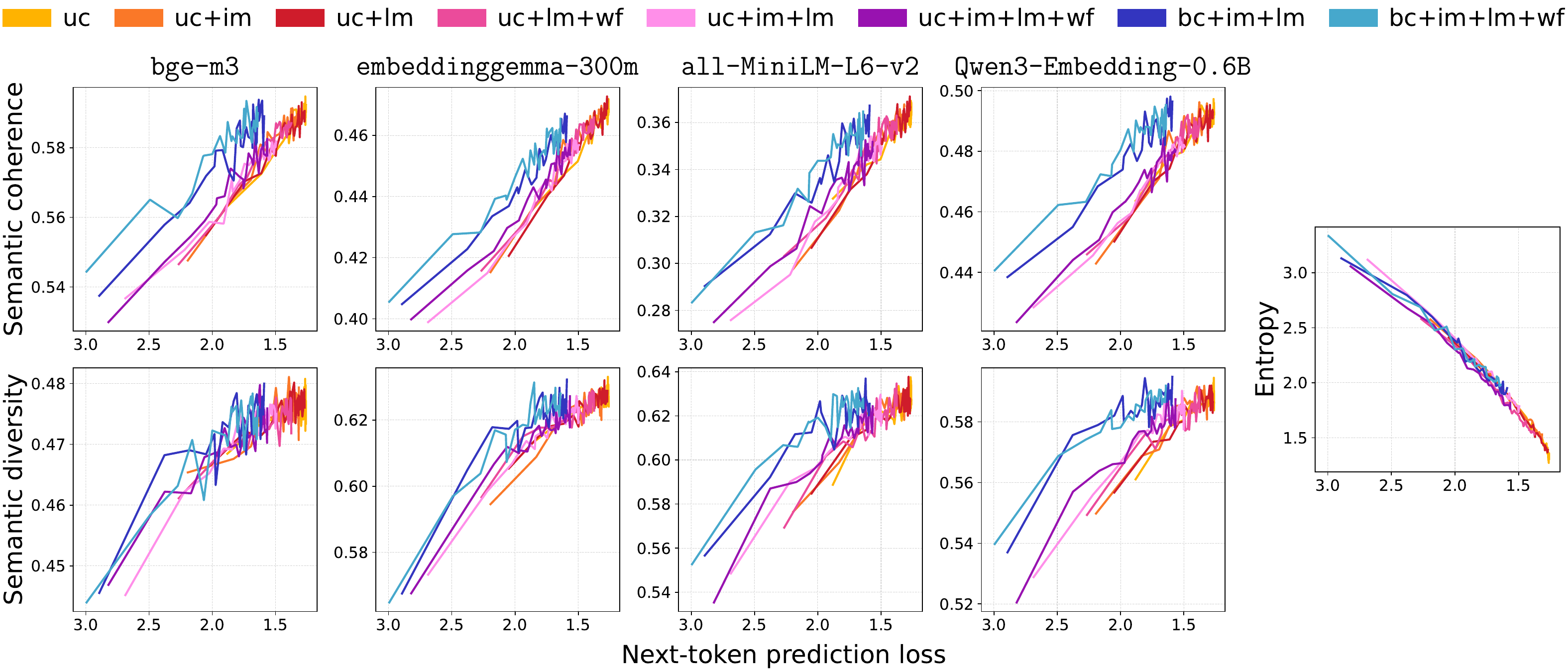}
    \end{center}
    \caption{Controlled experiments on interpolated training objectives using the TinyStories dataset, presented in the same format as \Fref{fig:trainingresults}.}
    \label{fig:tinystorytrain}
\end{figure}

\clearpage

\section{Supplementary analyses of decoding algorithms}

\subsection{Proof of entropy reduction}
\label{app:inference_theory}
We formally prove that confidence-based strategies reduce full entropy, assuming the data distribution factorizes into independent position-wise distributions, and the DLM is optimally trained.

Throughout this section, we use superscripts to index positions within a sequence. To contextualize the proof, we briefly recall the decoding mechanisms of confidence-based strategies. We assume a block length $B$ and a target number of denoising steps $N$. Let $p_\theta$ be the trained DLM. At each denoising step, given current partially decoded sequence $x_t$, the model samples a token from $p_\theta(x_0^i \mid x_t)$ for each masked position $i$ within the block. 

Low-confidence remasking decodes $B/N$ tokens within this block that have the highest prediction probabilities $p_\theta(x_0^i \mid x_t)$ (i.e. the confidence values), remasking the others. Dynamic low-confidence remasking extends this by additionally decoding any token in the block whose confidence value strictly exceeds a predefined threshold $\tau$, meaning that the number of tokens decoded per step can dynamically increase. Consistent with our experimental design, we set $N=B$, which targets decoding one token per step. Once all tokens within the current block are decoded, the generation process proceeds to the next block.

\begin{theorem}
\label{theorem:appendix}
Assume that the data distribution at each position is independent of the others, so that the full distribution factorizes as
\begin{equation}
p_{\mathrm{data}}(x)=\prod_{i=1}^{L} p_{\mathrm{data}}^i(x^i)
=: \prod_{i=1}^{L} q^i_{x^i}\ ,
\end{equation}
and that the DLM is optimally trained, i.e., for any masked position $i$ of a noised input $x_t$,
\begin{equation}
p_\theta(x_0^i \mid x_t)=p_{\mathrm{data}}^i(x_0^i)=q^i_{x_0^i}.
\end{equation}
Let $p_{\mathrm{lcr}}$ and $p_{\mathrm{dlcr}}$ denote the generation distributions induced by low-confidence remasking and dynamic low-confidence remasking, respectively. Then
\begin{equation}
\mathcal H(p_{\mathrm{lcr}})\le \mathcal H(p_{\mathrm{data}}), \qquad
\mathcal H(p_{\mathrm{dlcr}})\le \mathcal H(p_{\mathrm{data}}).
\end{equation}
\end{theorem}

\begin{proof}
Without loss of generality, we assume that $q_1^i \ge q_2^i \ge \cdots \ge q_{|\mathcal V|}^i$ for all $i$. By definition, low-confidence remasking is a special case of dynamic low-confidence remasking when $\tau=1$. Thus, it suffices to prove $\mathcal H(p_{\mathrm{dlcr}})\le \mathcal H(p_{\mathrm{data}})$.

We assume the model starts generating from an empty prompt, as our proof can generalize the starting state to any sequence.

Let \(X \sim p_{\mathrm{dlcr}}\), $X=(X^1,X^2,...,X^L)$. By the chain rule of entropy,
\begin{equation}
\mathcal H(p_{\mathrm{dlcr}})
= \mathcal H(X^1) + \mathcal H(X^2 \mid X^1) + \cdots + \mathcal H(X^L \mid X^{1:L-1}).
\end{equation}
It is sufficient to show that for every position \(i\) and possible prefix sequence $x^{1:i-1}$, we have
\begin{equation}
\label{eq:proof_1}
\forall k,\ \sum_{c=1}^{k} p_{\mathrm{dlcr}}(X^i=c \mid X^{1:i-1}=x^{1:i-1})
\ge
\sum_{c=1}^{k} q_c^i.
\end{equation}
This is a majorization relation~\citep{marshall1979inequalities}. Since the entropy function is Schur-concave, the majorization implies
\begin{equation}
\mathcal H(X^i \mid X^{1:i-1}=x^{1:i-1}) \le \mathcal H(q^i)=\mathcal H(p_{\mathrm{data}}^i),
\end{equation}
and taking expectation over \(x^{1:i-1}\) gives
\begin{equation}
\mathcal H(X^i \mid X^{1:i-1}) \le \mathcal H(p_{\mathrm{data}}^i).
\end{equation}
Summing over \(i\) then yields our desired inequality
\begin{equation}
\mathcal H(p_{\mathrm{dlcr}}) \le \mathcal H(p_{\mathrm{data}}).
\end{equation}

We prove \mbox{Eq.\eqref{eq:proof_1}} by introducing latent variable \(Z\), where \(Z^i\) denotes the decoding step at which \(X^i\) is generated. It is sufficient to show that for any configuration \(z^1,\ldots,z^L\) and any realization of the other tokens \(x^1,\dots,x^{i-1},x^{i+1},\dots,x^L\) such that
\begin{equation}
p_{\mathrm{dlcr}}(X^{1:i-1}=x^{1:i-1},X^{i+1:L}=x^{i+1:L},Z^{1:L}=z^{1:L})>0,
\end{equation}
we have for all \(k\),
\begin{equation}
\label{eq:proof_3}
\sum_{c=1}^{k}
p_{\mathrm{dlcr}}\!\left(
X^i=c \,\middle|\,
X^{1:i-1}=x^{1:i-1},X^{i+1:L}=x^{i+1:L},Z^{1:L}=z^{1:L}
\right)
\ge
\sum_{c=1}^{k} q_c^i.
\end{equation}
This is sufficient because the conditional distribution given \(X^{1:i-1}\) can be expressed as a mixture over the remaining variables, and majorization is preserved under such mixtures.

Let
\begin{equation}
S=\{j:z^j=z^i\}, \qquad T=\{j:(z^j\ge z^i) \land (j \text{ and } i \text{ are in the same block})\},
\end{equation}
be the set of positions decoded at step \(z^i\) and the set of positions within the same block as \(i\) that are not decoded before step \(z^i\). We consider the following two cases.

\paragraph{Case 1: \(|S|=1\).}
In this case, the decoding step of \(X^i\) decodes exactly one position, namely \(i\). The event that \(X^i\) is the only token that is decoded at this step occurs if and only if all positions in \(T\setminus\{i\}\) are remasked at the same step. Under the optimal model assumption, decoding \(X^i=c\) requires that for every \(j \in T \setminus \{i\}\), the sampled token at position \(j\) has confidence no larger than \(q_c^i\) and does not exceed the threshold \(\tau\). We claim that
\begin{equation}
\label{eq:proof_2}
\begin{gathered}
p_{\mathrm{dlcr}}\!\left(
X^i=c \,\middle|\,
X^{1:i-1}=x^{1:i-1},X^{i+1:L}=x^{i+1:L},Z^{1:L}=z^{1:L}
\right)\\
\propto
q_c^i
\prod_{j\in T\setminus\{i\}}
\sum_{r:\, q_r^j \le \min(\tau,q_c^i)} q_r^j.
\end{gathered}
\end{equation}
To understand this proportionality, note that by Bayes' theorem, the conditional probability is proportional to the joint probability:
\begin{equation}
\begin{gathered}
p_{\mathrm{dlcr}}\!\left(
X^i=c \,\middle|\,
X^{1:i-1}=x^{1:i-1},X^{i+1:L}=x^{i+1:L},Z^{1:L}=z^{1:L}
\right)\\
\propto
p_{\mathrm{dlcr}}\!\left(
X^i=c ,
X^{1:i-1}=x^{1:i-1},X^{i+1:L}=x^{i+1:L},Z^{1:L}=z^{1:L}
\right).
\end{gathered}
\end{equation}
Expanding this joint probability into a product of step-wise probabilities reveals that most of the factors are independent of $c$ and can be absorbed into the normalization constant. The only terms dependent on $c$ occur at the specific step where $X^i$ is decoded: (1) the probability $q_c^i$ of sampling $X^i=c$, and (2) the probability that tokens sampled at the remaining positions $T\setminus\{i\}$ fall below $\min(\tau, q_c^i)$ and are thus remasked.

From \mbox{Eq.\eqref{eq:proof_2}}, define
\begin{equation}
\label{eq:wc}
w_c :=
\prod_{j\in T\setminus\{i\}}
\sum_{r:\, q_r^j \le \min(\tau,q_c^i)} q_r^j.
\end{equation}
Since \(q_c^i\) is non-increasing in \(c\), the weight \(w_c\) is also non-increasing in \(c\). Therefore, after normalization, the distribution
\begin{equation}
p_{\mathrm{dlcr}}\!\left(
X^i=c \,\middle|\,
X^{1:i-1}=x^{1:i-1},X^{i+1:L}=x^{i+1:L},Z^{1:L}=z^{1:L}
\right) \propto q_c^i w_c
\end{equation}
majorizes \(q^i\), i.e.,
\begin{equation}
\forall k,\ \sum_{c=1}^{k} p_{\mathrm{dlcr}}\!\left(
X^i=c \,\middle|\,
X^{1:i-1}=x^{1:i-1},X^{i+1:L}=x^{i+1:L},Z^{1:L}=z^{1:L}
\right) \ge \sum_{c=1}^{k} q_c^i.
\end{equation}

\paragraph{Case 2: \(|S|>1\).}
In this case, the decoding step of \(X^i\) simultaneously decodes multiple positions. Thus \(X^i\) is decoded to \(c\) if and only if its confidence \(q_c^i\) exceeds the threshold \(\tau\). Hence,
\begin{equation}
\begin{gathered}
p_{\mathrm{dlcr}}\!\left(
X^i=c \,\middle|\,
X^{1:i-1}=x^{1:i-1},X^{i+1:L}=x^{i+1:L},Z^{1:L}=z^{1:L}
\right)\\
\propto
q_c^i \mathbf{1}[q_c^i>\tau].
\end{gathered}
\end{equation}
As before, this proportionality holds because expanding the joint probability reveals that most step-wise factors are independent of $c$ and absorbed into the normalization constant. The only terms dependent on $c$ occur at the specific step where $X^i$ is decoded: (1) the probability $q_c^i$ of sampling $c$, and (2) the threshold constraint $\mathbf{1}[q_c^i > \tau]$.

Let
\begin{equation}
w_c := \mathbf{1}[q_c^i>\tau].
\end{equation}
Since \(q_c^i\) is non-increasing in \(c\), the weight \(w_c\) is also non-increasing in \(c\). Therefore, after normalization, the distribution
\begin{equation}
p_{\mathrm{dlcr}}\!\left(
X^i=c \,\middle|\,
X^{1:i-1}=x^{1:i-1},X^{i+1:L}=x^{i+1:L},Z^{1:L}=z^{1:L}
\right)\propto q_c^i w_c
\end{equation}
majorizes \(q^i\), i.e.,
\begin{equation}
\forall k,\ \sum_{c=1}^{k} p_{\mathrm{dlcr}}\!\left(
X^i=c \,\middle|\,
X^{1:i-1}=x^{1:i-1},X^{i+1:L}=x^{i+1:L},Z^{1:L}=z^{1:L}
\right) \ge \sum_{c=1}^{k} q_c^i.
\end{equation}

Combining the two cases, we obtain \mbox{Eq.\eqref{eq:proof_3}}, completing the proof.
\end{proof}

\textbf{Remark.} Ambiguity may arise when there exist $i,j$ and $i',j'$ such that $q_j^i = q_{j'}^{i'}$, as it is not defined which token takes priority for remasking. However, this is not a significant issue: first, such exact equality is extremely rare in practice; second, even if it occurs, as long as we assign an oracle order to these tokens, \mbox{Eq.\eqref{eq:wc}} remains non-increasing in $c$, and the rest of the proof remains unaffected and valid.

\newpage

\subsection{Bias elimination experiments}
\label{app:inference_bias}

In \Sref{sec:inference_results}, we suggest that the entropy reduction observed in confidence-based remasking strategies stems from a theoretical distributional bias. To support this explanation, we introduce a simple modification to the confidence-based strategy: after selecting the decoding positions based on confidence values, we \textbf{resample} the decoded tokens at these positions. This breaks the correlation between position selection and the sampled tokens, thereby eliminating the bias. 

As shown in \Fref{fig:inference_entropy}, this modification restores $n$-gram entropy to the baseline level and invalidates inequality~(\hyperref[eq:inference_inq]{ii}) in \mbox{Eq.\eqref{eq:inference_inq}}. This confirms that the bias causes the $n$-gram entropy reduction. 
These phenomena are consistent across seeds, datasets and architectures (\Fref{fig:finewebinference_entropy}, \Fref{fig:qweninference_entropy} and \Fref{fig:tinystoryinference_entropy}).

Beyond entropy, bias elimination also removes the influences of confidence-based strategies on semantic coherence and semantic diversity (\Fref{fig:finewebinference_semantic}, \Fref{fig:qwen2inference_semantic}, and \Fref{fig:tinystoryinference_semantic}). This suggests that the bias-based distinction could be generalized to these metrics as well.

\subsection{Ablations across seeds, datasets and architectures}
\label{app:inference_robust}

Consistent with Appendix~\ref{app:trainrobust}, we further validate our main findings by repeating the main experiments with two additional random seeds, using the Qwen2 architecture, and on the TinyStories dataset. As shown in \Fref{fig:finewebinference_entropy}, \Fref{fig:qweninference_entropy}, and \Fref{fig:tinystoryinference_entropy}, we consistently observe entropy reduction, including decreases in $n$-gram entropy and the validity of inequality~(\hyperref[eq:inference_inq]{ii}). We also observe that bias elimination remains effective. These results confirm the robustness of our findings.

\subsection{Effects on semantic coherence and semantic diversity}
\label{app:inference_other}

We also evaluate semantic coherence and semantic diversity across all remasking strategies and configurations. As shown in \Fref{fig:finewebinference_semantic}, \Fref{fig:qwen2inference_semantic}, and \Fref{fig:tinystoryinference_semantic}, non-confidence-based strategies have much smaller effects on both metrics compared to confidence-based strategies. Confidence-based strategies consistently increase semantic coherence, while their effect on semantic diversity depends highly on the dataset. As shown in \Fref{fig:finewebinference_semantic} and \Fref{fig:qwen2inference_semantic}, they increase semantic diversity on the FineWeb dataset, consistent across seeds and architectures. However, when we switch to the TinyStories dataset (\Fref{fig:tinystoryinference_semantic}), confidence-based strategies reduce semantic diversity.

\clearpage

\begin{figure}[p]
    \begin{subfigure}{\linewidth}
        \centering
        \includegraphics[width=\linewidth]{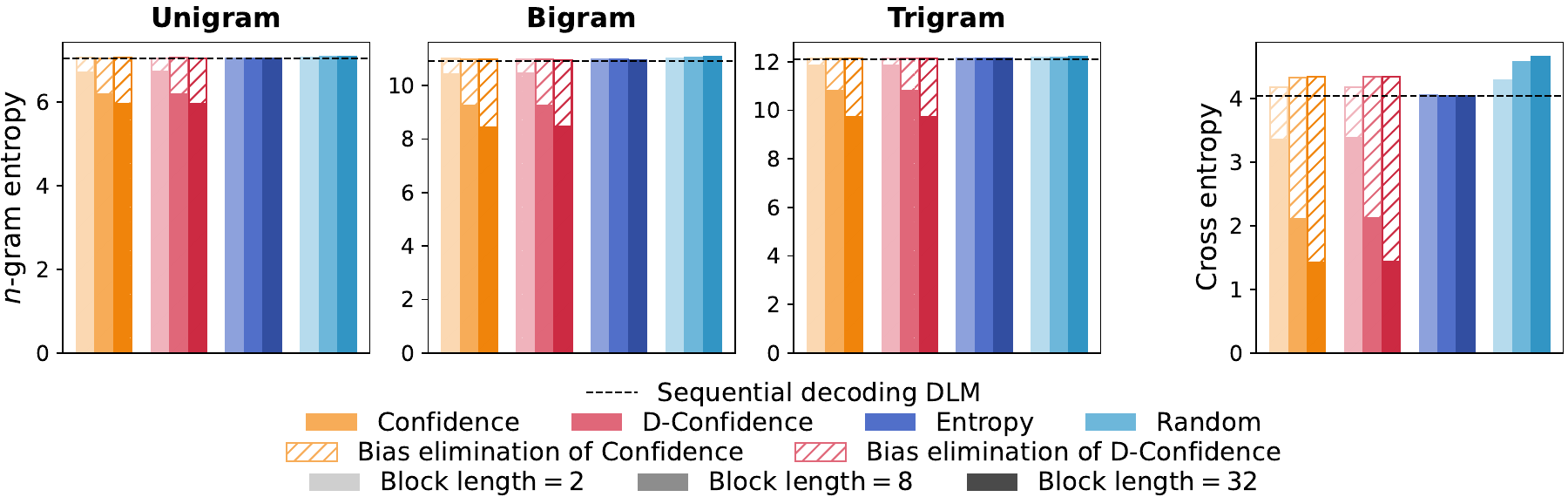}
        \caption{seed = 1}
    \end{subfigure}
    
    \vspace{1em}
    
    \begin{subfigure}{\linewidth}
        \centering
        \includegraphics[width=\linewidth]{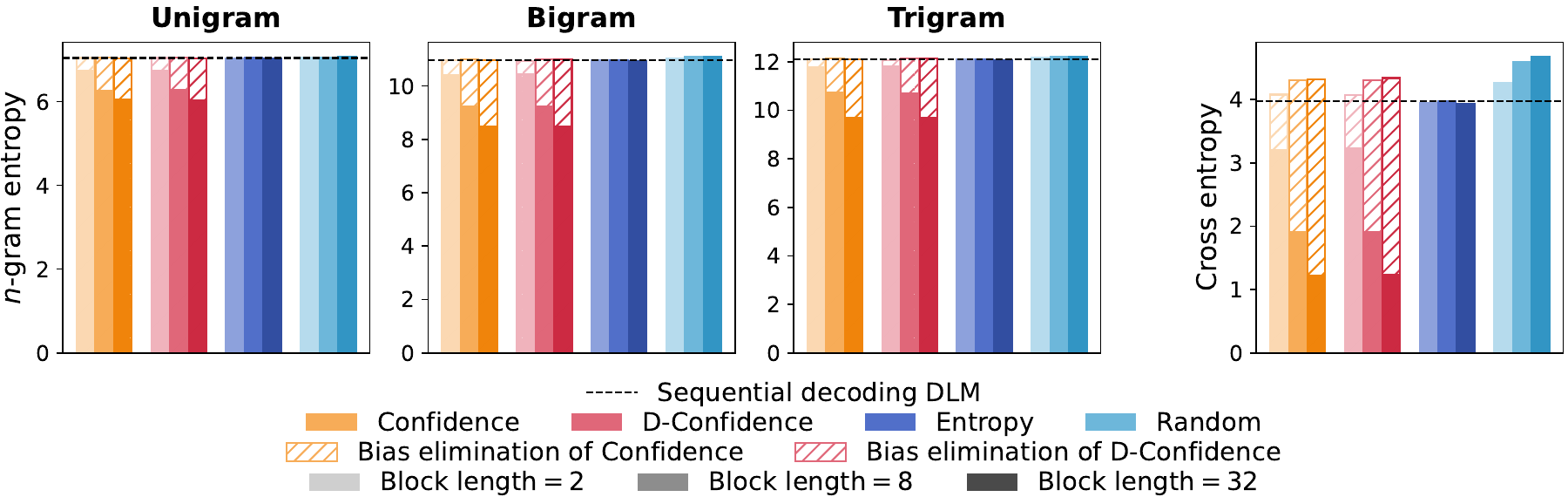}
        \caption{seed = 2}
    \end{subfigure}
    
    \caption{$n$-gram entropy and cross-entropy (defined in \mbox{Eq.\eqref{eq:ce}}) across different remasking strategies with other seeds. The format follows \Fref{fig:inference_entropy}, which displays the seed = 0 results.}
    \label{fig:finewebinference_entropy}
\end{figure}

\clearpage

\begin{figure}[p]
    \begin{center}
    \includegraphics[width=\linewidth]{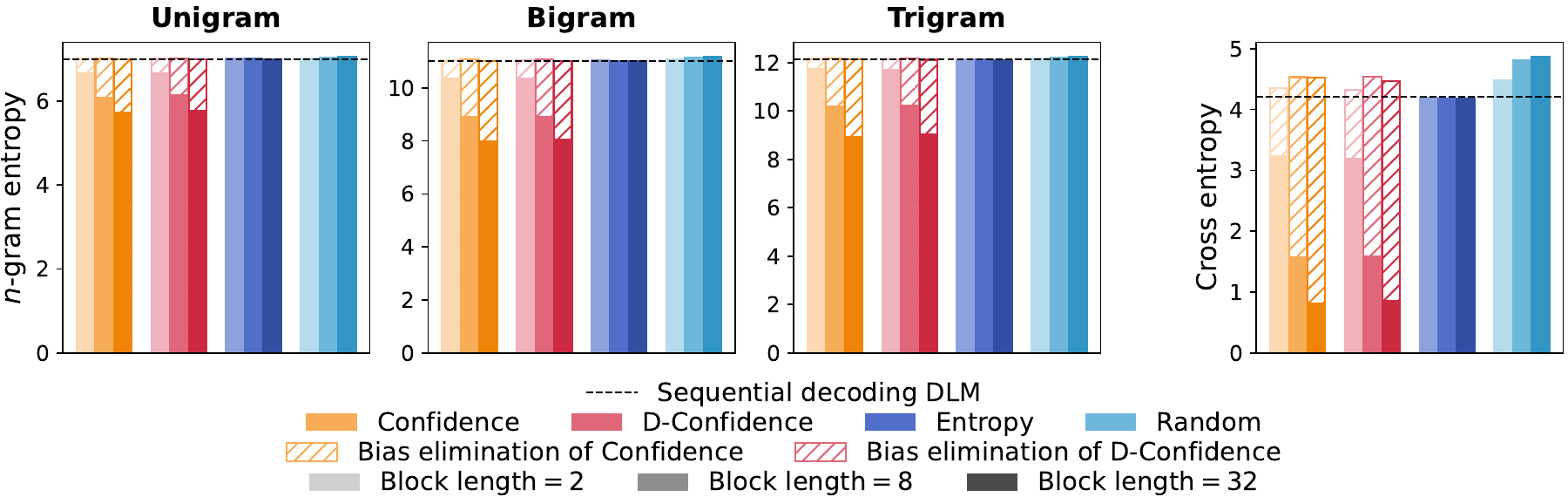}
    \end{center}
    \caption{$n$-gram entropy and cross-entropy (defined in \mbox{Eq.\eqref{eq:ce}}) across different remasking strategies for the Qwen2 architecture experiment, presented in the same format as \Fref{fig:inference_entropy}.}
    \label{fig:qweninference_entropy}
\end{figure}

\begin{figure}[p]
    \begin{center}
    \includegraphics[width=\linewidth]{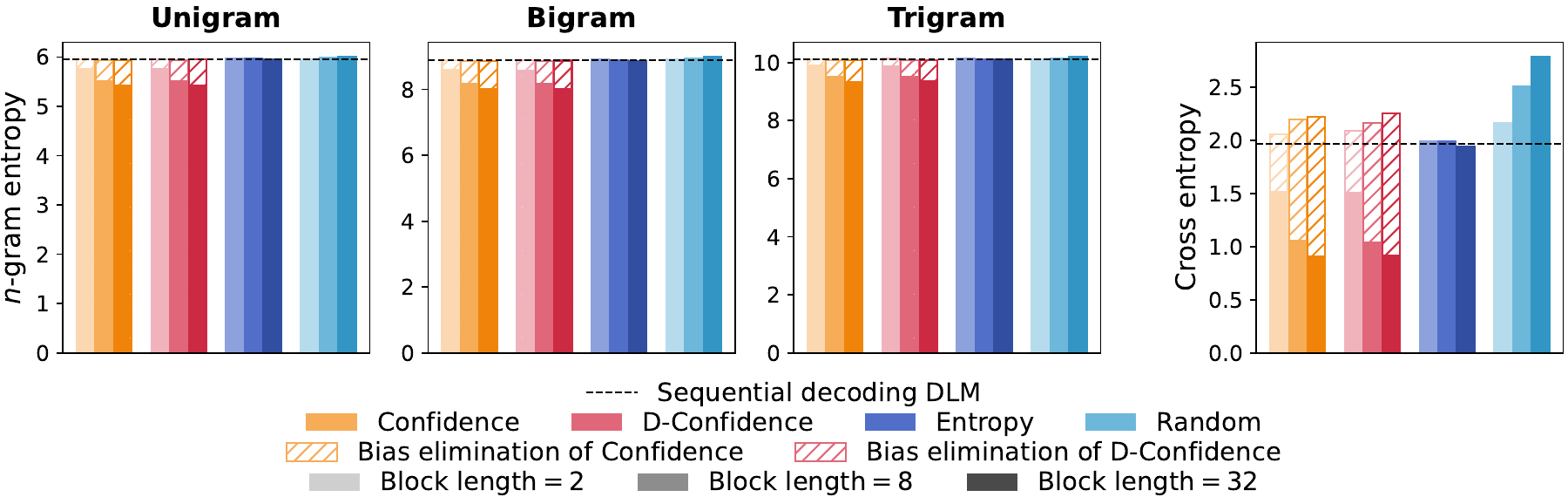}
    \end{center}
    \caption{$n$-gram entropy and cross-entropy (defined in \mbox{Eq.\eqref{eq:ce}}) across different remasking strategies for the TinyStories dataset experiment, presented in the same format as \Fref{fig:inference_entropy}.}
    \label{fig:tinystoryinference_entropy}
\end{figure}

\newpage

\begin{figure}[htb]
    \begin{subfigure}{\linewidth}
        \centering
        \includegraphics[width=\linewidth]{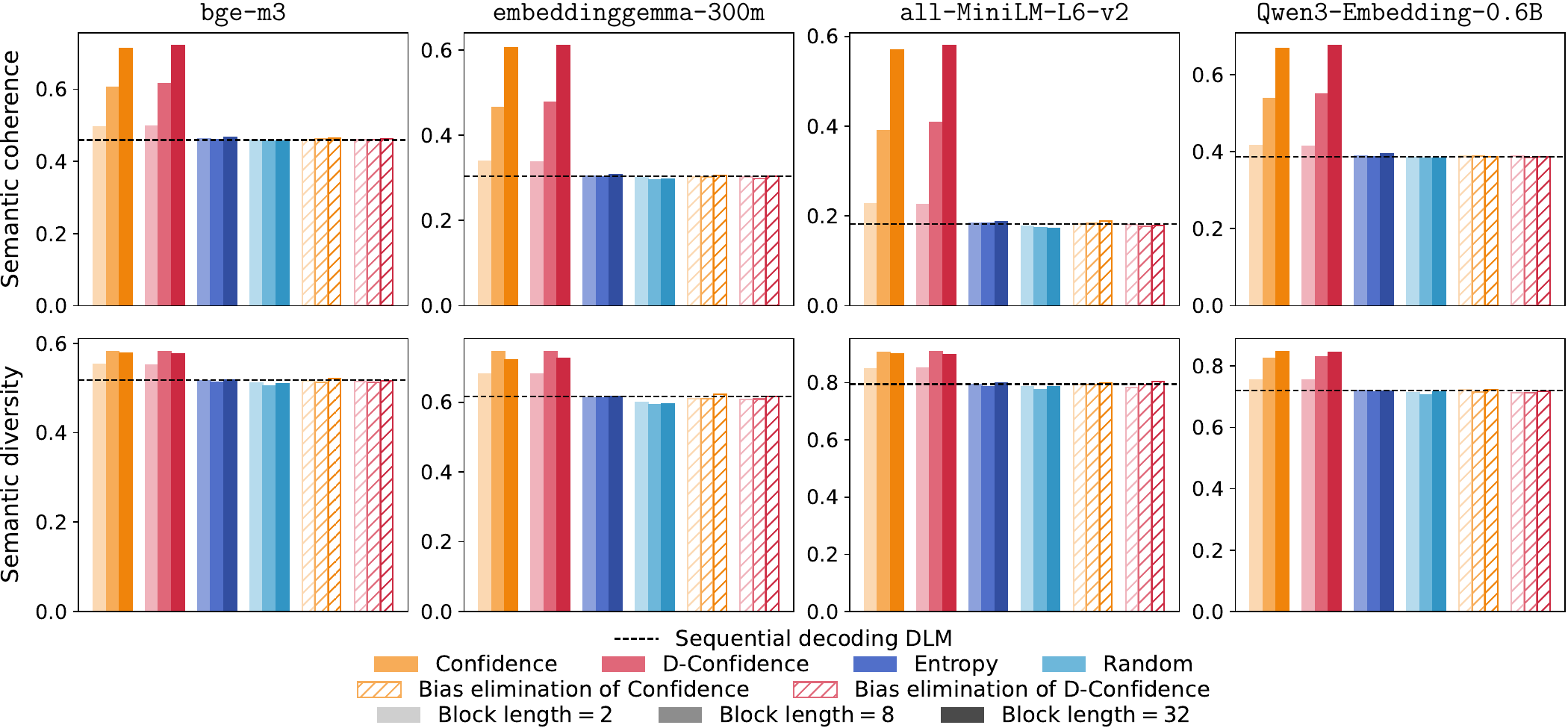}
        \caption{seed = 0}
    \end{subfigure}
    
    \vspace{1em}
    
    \begin{subfigure}{\linewidth}
        \centering
        \includegraphics[width=\linewidth]{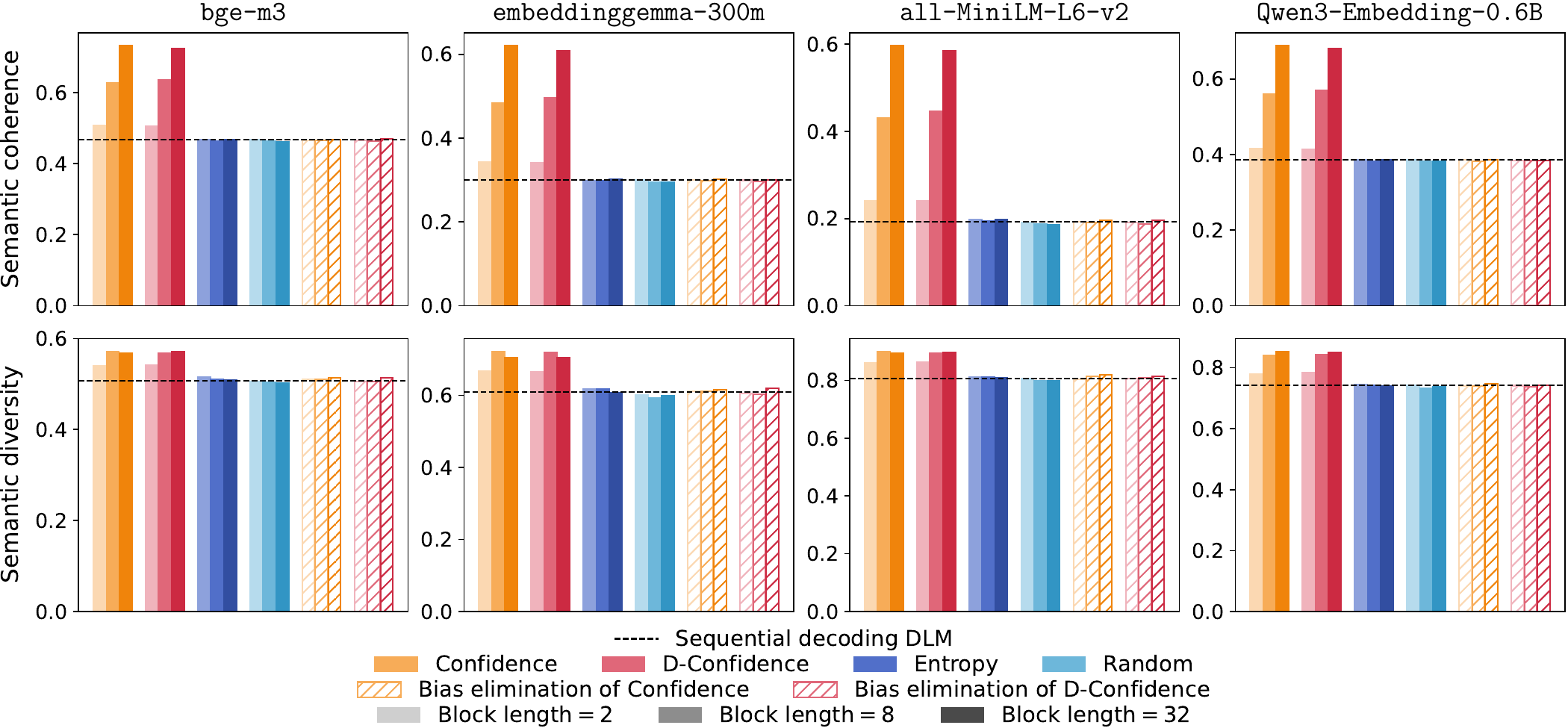}
        \caption{seed = 1}
    \end{subfigure}
    
    \vspace{1em}
    
    \begin{subfigure}{\linewidth}
        \centering
        \includegraphics[width=\linewidth]{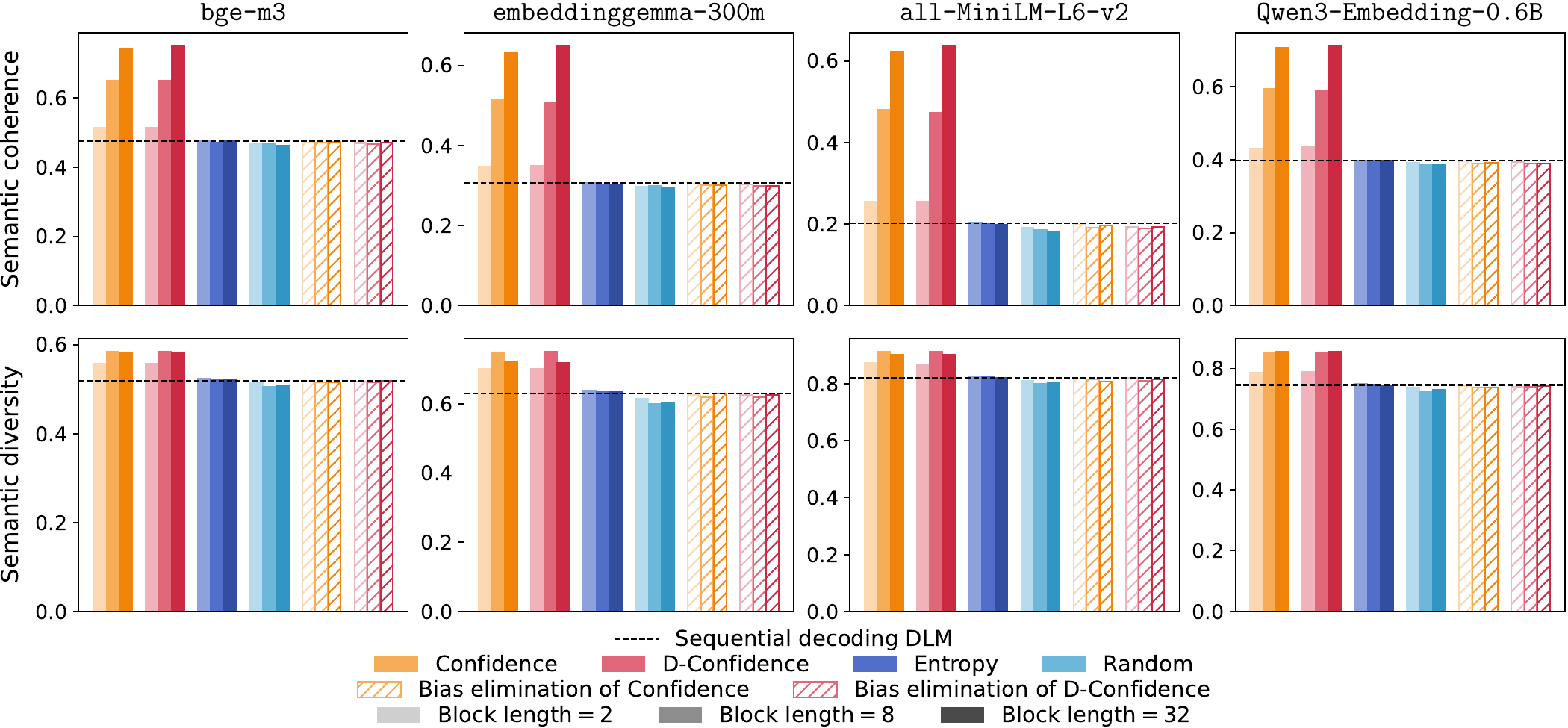}
        \caption{seed = 2}
    \end{subfigure}
    
    \caption{Semantic coherence and semantic diversity across different remasking strategies with three random seeds.}
    \label{fig:finewebinference_semantic}
\end{figure}

\begin{figure}[htb]
    \begin{center}
    \includegraphics[width=\linewidth]{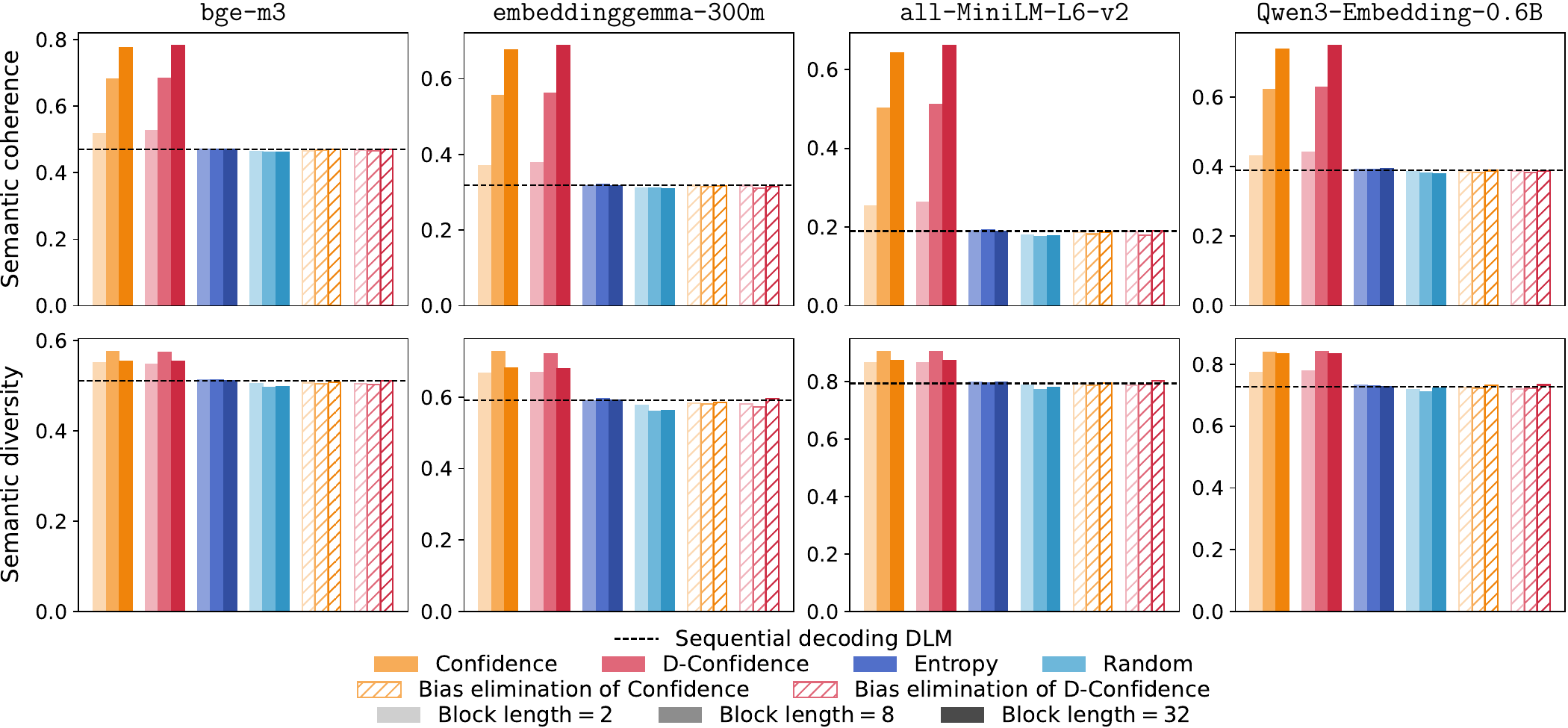}
    \end{center}
    \caption{Semantic coherence and semantic diversity across different remasking strategies for the Qwen2 architecture experiment.}
    \label{fig:qwen2inference_semantic}
\end{figure}

\begin{figure}[htb]
    \begin{center}
    \includegraphics[width=\linewidth]{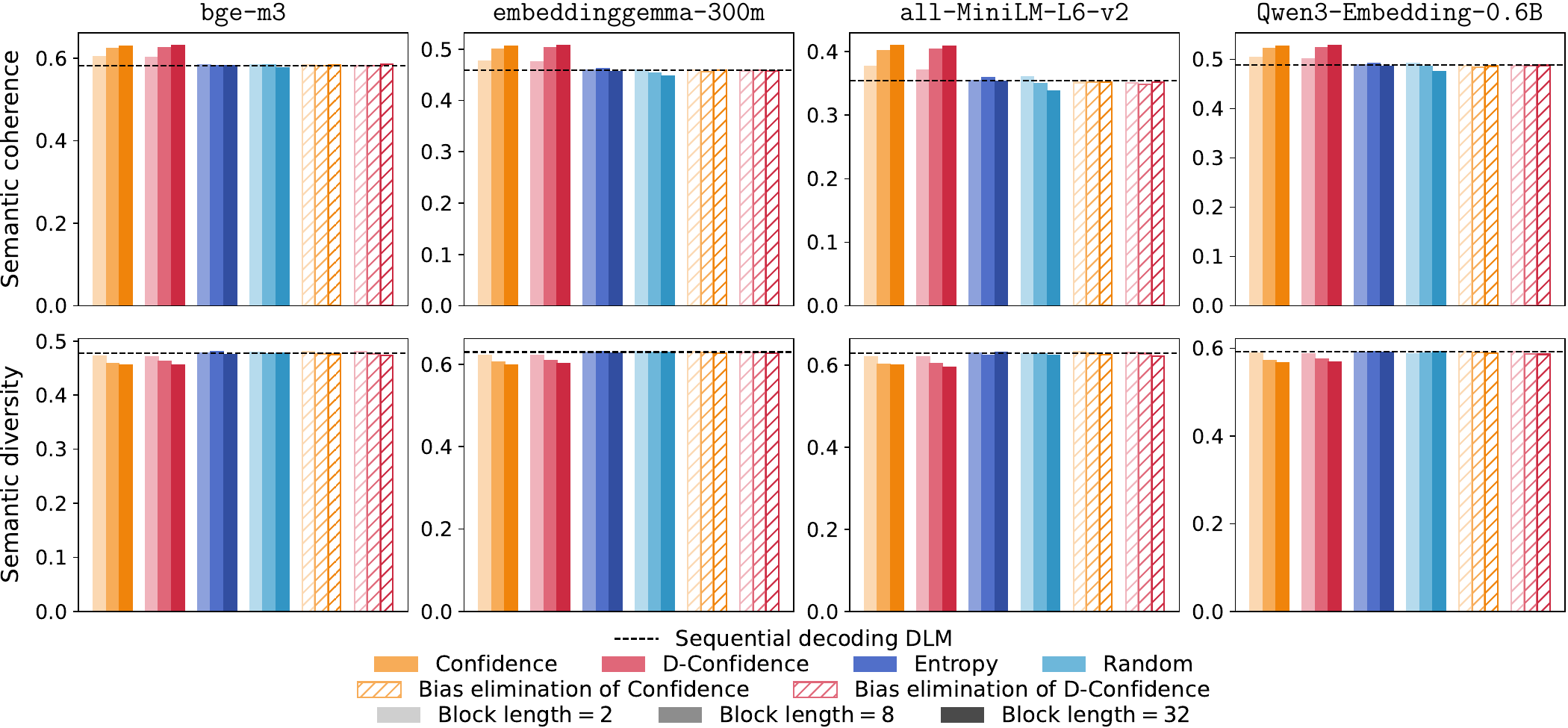}
    \end{center}
    \caption{Semantic coherence and semantic diversity across different remasking strategies for the TinyStories dataset experiment.}
    \label{fig:tinystoryinference_semantic}
\end{figure}

\end{document}